\documentclass{article} % For LaTeX2e
\usepackage{iclr2027_conference,times}

\usepackage{amsmath,amsfonts,bm}

\def\eqref#1{equation~\ref{#1}}
\def\1{\bm{1}}

\DeclareMathAlphabet{\mathsfit}{\encodingdefault}{\sfdefault}{m}{sl}
\SetMathAlphabet{\mathsfit}{bold}{\encodingdefault}{\sfdefault}{bx}{n}

\usepackage[colorlinks=true,linkcolor=NavyBlue,citecolor=NavyBlue,urlcolor=NavyBlue]{hyperref}
\usepackage{url}

\usepackage{amsmath,amssymb,amsthm}
\newtheorem{proposition}{Proposition}
\newtheorem{lemma}{Lemma}
\usepackage{multirow}
\usepackage{booktabs}
\usepackage{enumitem}
\usepackage{graphicx}
\usepackage{xcolor}
\usepackage{colortbl}
\usepackage{microtype}
\usepackage{tikz}
\usepackage{capt-of}
\usepackage{stmaryrd}
\usepackage[capitalise,noabbrev,nameinlink]{cleveref}
\usetikzlibrary{arrows.meta,positioning}
\usepackage{pifont}
\usepackage{titletoc}
\newcommand{\cmark}{\textcolor{green!60!black}{\ding{51}}}
\newcommand{\xmark}{\textcolor{red!80!black}{\ding{55}}}
\usetikzlibrary{calc}

\definecolor{NavyBlue}{RGB}{0,60,120}
\definecolor{DiagFrozen}{RGB}{237,239,242}
\definecolor{DiagFrozenL}{RGB}{170,176,186}
\definecolor{DiagTrain}{RGB}{219,233,247}
\definecolor{DiagTrainL}{RGB}{ 92,141,196}
\definecolor{DiagData}{RGB}{230,243,232}
\definecolor{DiagDataL}{RGB}{104,160,116}
\definecolor{DiagEmb}{RGB}{251,236,225}
\definecolor{DiagEmbL}{RGB}{206,124, 72}

\tikzset{
  dbox/.style   ={draw, rounded corners=2.5pt, minimum height=8mm, align=center, inner sep=3.5pt,
                  font=\small},
  frozen/.style ={dbox, fill=DiagFrozen, draw=DiagFrozenL},
  trained/.style={dbox, fill=DiagTrain,  draw=DiagTrainL, very thick},
  ddata/.style  ={dbox, fill=DiagData,   draw=DiagDataL},
  demb/.style   ={dbox, fill=DiagEmb,    draw=DiagEmbL},
  darr/.style   ={-{Stealth[length=2.2mm]}, thick, draw=black!65},
}

\newcommand{\code}[1]{\texttt{\small #1}}

\newcommand{\Dth}{D_\theta}
\newcommand{\Dtil}{\widetilde{D}}
\newcommand{\xhat}{\hat{x}_0}
\newcommand{\hs}{h}
\newcommand{\adapt}{\varphi}
\newcommand{\textenc}{E_\vartheta^{\mathrm{txt}}}
\newcommand{\imgenc}{E_\vartheta^{\mathrm{img}}}

\newcommand{\sg}[1]{\mathrm{sg}\!\left[#1\right]}
\newcommand{\FD}{\mathrm{FD}}
\newcommand{\MIND}{\mathrm{MIND}}
\newcommand{\FID}{\mathrm{FID}}
\newcommand{\pstar}{p^{*}}
\newcommand{\avgknn}{\mathrm{AvgkNN}}
\newcommand{\method}{Debias Anything}
\definecolor{Grad}{RGB}{160,196,236}
\newcommand{\shade}[2]{\cellcolor{Grad!#1}#2}                 % per-column gradient, 100 = best
\newcommand{\best}[1]{\cellcolor{Grad}\textbf{\boldmath #1}}  % best guided value per column
\newcommand{\rowmark}[1]{\tikz[remember picture,overlay]\coordinate(#1);}
\newcommand{\rowbracket}[2]{\begin{tikzpicture}[remember picture,overlay]
  \draw[NavyBlue,thick] ([xshift=-2pt,yshift=2.0ex]#1) -- ([xshift=-6pt,yshift=2.0ex]#1)
    -- ([xshift=-6pt,yshift=-1.1ex]#2) -- ([xshift=-2pt,yshift=-1.1ex]#2);
\end{tikzpicture}}
\newcommand{\lo}[1]{#1$\downarrow$}
\newcommand{\hi}[1]{#1$\uparrow$}
\newcommand{\rowfig}[4]{\begin{tikzpicture}
  \node[anchor=south west, inner sep=0] (img) {\includegraphics[width=#1, trim=#3 0 0 0, clip]{#2}};
  \foreach \f/\lab in {#4}
    \node[font=\scriptsize, rotate=90, anchor=south, align=center] at ($(img.south west)!\f!(img.north west)+(-0.8mm,0)$) {\lab};
\end{tikzpicture}}
\newcommand{\Legendbest}{Guided rows are shaded from best (dark, bold) to worst (light)}

\tikzset{
  fdaxis/.style={black!65},
  fdgrid/.style={black!12},
  fdtick/.style={font=\footnotesize, color=black!75, inner sep=2pt},
  fdwlab/.style={font=\tiny, color=black!55, inner sep=1.2pt},
  fdname/.style={font=\small, align=center},
  fdpt/.style={circle, inner sep=1.1pt},
  fdbase/.style={rectangle, inner sep=1.8pt, fill=white, draw=black!70, very thick},
  fdleg/.style={font=\tiny, color=black!70, inner sep=1pt},
}

\title{Debias Anything: Fairness with Diversity without Supervision in Diffusion Models}

\author{Théau d'Audiffret\textsuperscript{1,2,3}, Mariia Vladimirova\textsuperscript{1,4,}\footnotemark[1]\hphantom{\textsuperscript{*}}, Jean-Yves Franceschi\textsuperscript{1,}\thanks{Equal contribution.} \\
\textsuperscript{1}Criteo AI Lab, Paris, France \textsuperscript{2}ESSEC Business School, France \\ \textsuperscript{3}CentraleSupélec, Université Paris-Saclay, France \textsuperscript{4}FairPlay joint team
}

\iclrfinalcopy % Uncomment for camera-ready version, but NOT for submission.
\begin{document}

\maketitle

\begin{abstract}

Although diffusion models produce high-quality images, they also reproduce and amplify demographic imbalances in their training data. Debiasing their generation process post-training w.r.t.\ some sensitive attribute usually relies on classifier guidance or explicit text extra-conditioning, but this reduces methods' applicability and output diversity. Conversely, methods promoting diversity alone do not ensure fair attribute representation. In this paper, we propose a method tackling fairness and diversity jointly that is generally applicable to any diffusion model and any sensitive attribute. To this end, an adapter connects the frozen diffusion model to a pretrained vision-language embedding space, enabling fairness and diversity guidance without sensitive-attribute annotations. For fairness, pairs of text prompts define attribute directions which guide batch composition towards specific proportions. For diversity, we introduce a score measuring disagreement between the semantic estimates derived from this representation. The formulation supports unconditional and text-conditional diffusion models, while requiring no prior knowledge or data of sensitive attribute. Experiments confirm that our method improves quality and diversity scores at comparable fairness levels.
Our code is available at \url{https://github.com/theaudaudiffret/Debias_anything}.

\end{abstract}

\begin{figure}[h]\centering
\rowfig{0.95\linewidth}{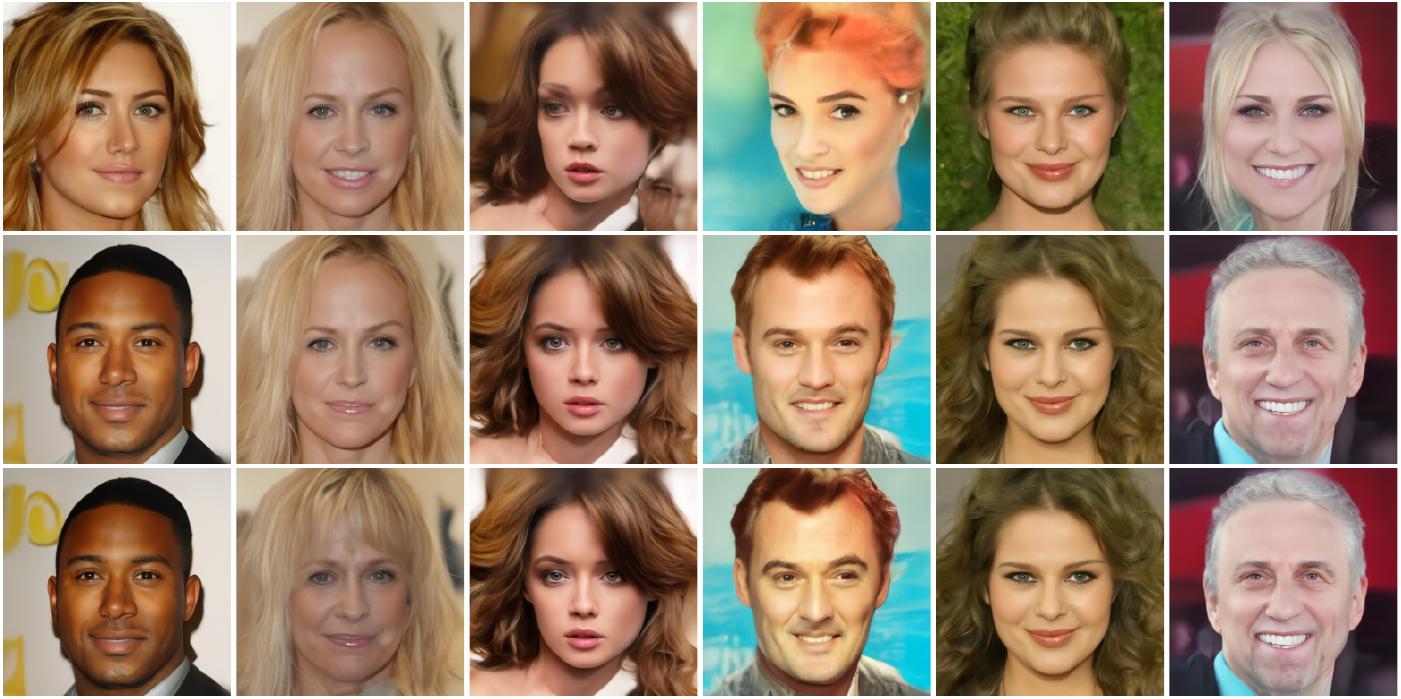}{0}{0.833/{Original}, 0.500/{\method{}}, 0.167/{\method{}\\$+$ Diversity}}
\caption{Generations examples on CelebA-HQ. \textit{Top}: the original model, 6 random draws of women; two of them appear with artifacts, the third looking away from the camera and the fourth turning her head. \textit{Middle}: guidance with the fairness term. Three columns turn into men and the faces sharpen,
but the pose collapses, every face now frontal and centered. \textit{Bottom}: guidance with the fairness and diversity terms. The woman in the second column looks older, the man in the fourth turns his head away and gains a mustache.}
\label{fig:teaser}
\end{figure}

\section{Introduction}
\label{sec:contribution}
 
Diffusion \citep{ho2020ddpm,song2021sde}, and equivalently flow matching \citep{lipman2023flow}, models are the standard unconditional generation models or text-to-image systems deployed at scale
\citep{rombach2022ldm,podell2024sdxl,wu2025qwen}.
 They inherit
imbalances of their training data and often exaggerate them
\citep{perera2023bias,bianchi2023easily,luccioni2023stablebias}, and the guidance they sample with
amplifies the imbalance further as its scale grows \citep{kim2026stayfair}.

Debiasing them includes choosing how often a given property, the \emph{sensitive attribute}, should appear
in the generated images, and making the model comply \citep{choi2020fairgen,teo2023measuring}.
Two problems make debiasing diffusion models a
hard problem. Firstly, existing methods have limited applicability with various restrictions: either they must rely on sensitive attribute supervision \citep{parihar2024balancing,shi2025difflens,zaleska2026icm}, require model finetuning \citep{shen2024finetuning}, are not adaptable to arbitrary attributes \citep{korada2025unsupervised}, or are restricted to text-to-image models \citep{friedrich2023fairdiffusion,zhang2023itigen}.
Secondly, guidance, used to steer a diffusion model towards fair representation as predicted by an attribute classifier, harms output diversity: the stronger the push, the more the samples cluster around a few typical
images and the rarer ones vanish \citep{ho2022cfg,sadat2024cads,sehwag2022lowdensity}. Methods that
restore diversity exist \citep{um2024sgms,corso2024particle}, but they do not aim for specific attribute fairness.

\begin{table}[t]\centering\footnotesize
\setlength{\tabcolsep}{2.5pt}
\caption{Positioning among debiasing methods for diffusion models. A method may need attribute
supervision (e.g.\ labels, reference images), may not accept an arbitrary target
proportion, may not have been applied to both unconditional and text-conditioned models, and may retrain the generator. 
\\ ($*$)
Score Guidance also presents an alternate supervision-free version, but with significantly dimi\-nished performance compared to the supervised version, which is the one we compare with.}
\label{tab:positioning}
\vspace{\baselineskip}
\scriptsize
\begin{tabular}{lccccl}
\toprule
 & \shortstack{no attribute\\supervision} & \shortstack{arbitrary\\proportion} & \shortstack{arbitrary\\conditioning} & \shortstack{no\\retraining} & attribute \\
\midrule
Fair Diffusion \citep{friedrich2023fairdiffusion} & \cmark & \cmark & \xmark & \cmark & text \\
ITI-GEN \citep{zhang2023itigen}                   & \xmark & \cmark & \xmark & \cmark & ref.\ images \\
Self-debiasing \citep{korada2025unsupervised}     & \cmark & \xmark & \cmark & \cmark & discovered \\
Latent editing \citep{kwon2023asyrp}              & \cmark & \xmark & \cmark & \cmark & text, training per attribute \\
Latent direction \citep{li2024selfdiscovering}    & \cmark & \cmark & \xmark & \cmark & text, training per attribute \\
FairGen \citep{jiang2025fairgen}                  & \cmark & \cmark & \xmark & \xmark & text, training per attribute \\
Fine-tuning \citep{shen2024finetuning}            & \xmark & \cmark & \xmark & \xmark & attribute-trained classifier \\
Balancing Act \citep{parihar2024balancing}        & \xmark & \cmark & \cmark & \cmark & labels \\
Inference-time alignment \citep{luan2026alignment} & \xmark & \cmark & \xmark & \cmark & attribute-trained classifier \\
Jeffrey guidance \citep{razafindralambo2026jeffrey} & \xmark & \cmark & \xmark & \cmark & attribute-trained classifier \\
DiffLens \citep{shi2025difflens}                  & \xmark & \cmark & \cmark & \cmark & labeled probes \\
ICM \citep{zaleska2026icm}                        & \xmark & \cmark & \xmark & \cmark & labeled probes \\
Score Guidance \citep{tiwary2026scoreguidance}    &\hphantom{*}\xmark/\cmark$^{*}$ & \cmark & \cmark & \cmark & ref.\ images\,/\,free text \\
\midrule
\method{} (Ours)                                         & \cmark & \cmark & \cmark & \cmark & free text \\
\bottomrule
\end{tabular}
\end{table}

We propose Debias Anything, a guidance that sets the proportion of any multi-class attribute one can describe in words, with no
labels and no generator retraining, and keeps the diversity of the model. The attribute is described independently of the diffusion model or prompt as short
sentences, for instance \emph{a person with eyeglasses} and \emph{a person without eyeglasses}. An
image-text encoder places sentences and images in one shared
space, so sentences pairs define a fairness direction which can be used to steer generated images towards fairness. To follow that direction while images are
still noisy, we train a small adapter, once per model and for no attribute in particular, that
predicts the encoder's expected embedding of the final image from the inner activations of the frozen
generator. Each sampling step then moves every image along a direction towards an assigned attribute class.
A second term, computed from the same adapter, pushes the images back towards the rare regions the
first term empties.

We compare \method{} with published baselines on unconditional face models and a text-to-image model, demonstrating its versatility. Despite the strong assumptions made by the baselines, especially regarding attribute supervision, \method{}
is better than or comparable to them. On CelebA-HQ, it reaches the best or near-best fairness score
on every considered attribute, with an improved image quality.
On Stable Diffusion~1.5, it substantially improves fairness across all considered attributes, with an image quality among the best. 
Finally, we show it can also target attributes for
which no label exists, confirming its attribute independence

% ===========================================================================
\section{Related work}
\label{sec:related-main}

\paragraph{Debiasing diffusion models.}
\cref{tab:positioning} compares existing methods. Almost all of them share at least one restriction among four. First, many need attribute supervision like reference images for ITI-GEN \citep{zhang2023itigen} and labeled images for DiffLens
\citep{shi2025difflens} and ICM \citep{zaleska2026icm}. Second, some cannot debias towards arbitrary attribute proportions \citep{korada2025unsupervised}. Third, most rely on text conditioning in the base model, which restricts them to text-to-image models \citep{friedrich2023fairdiffusion,li2024selfdiscovering,jiang2025fairgen}. Fourth, a few of them necessitate retraining or fine-tuning the base model \citep{shen2024finetuning}. Our method is free from all these assumptions.

The closest methods to ours steer a whole batch towards a target distribution
at sampling time. Balancing Act \citep{parihar2024balancing} does so through a predictor trained on
labeled U-Net activations, and every image receives the gradient of a batch-level distance, but remains a supervised method.
Score Guidance \citep{tiwary2026scoreguidance} needs exemplar images of each value, except in a text
version that needs no supervision but is less effective in most of the reported experiments. Both versions
score the attribute in pixel space, whereas our score is read in the encoder's embedding space (\cref{sec:adapter}), and both adopt a
greedy assignment of images to values early in sampling, whereas ours is optimally
recomputed at every step (\cref{prop:quota}).

\paragraph{Restoring diversity.}
Guidance narrows the distribution it samples \citep{ho2022cfg} and amplifies attribute imbalance as
its scale grows \citep{kim2026stayfair}. Sampling-time methods recover the tail by pushing towards
low-density regions of a feature space \citep{sehwag2022lowdensity}, repelling the samples of a batch
in pixel or feature space \citep{corso2024particle,bhat2026dontsettle}, maximising their kernel
entropy \citep{jalali2025sparke}, boosting the initial noise \citep{um2025boost}, annealing the
conditioning \citep{sadat2024cads}, measuring rarity against
a prior independent of the generator \citep{park2026jepa}, or scoring rarity by how much re-denoising
deforms the Tweedie estimate, with a trained classifier \citep{um2024minority} or without one
\citep[SGMS;][]{um2024sgms}. None of them constrains a sensitive attribute. Our rarity term transposes
SGMS into the encoder's space through the same adapter for efficiency purposes, and it is added to the fairness term so that
the tail is recovered at the target proportion.

\paragraph{Inferring the attribute during sampling.}
Classifier guidance \citep{dhariwal2021beatgans} trains a classifier on noisy images for each
attribute; training-free guidance \citep{chung2023dps,yu2023freedom,bansal2024universal,ye2024tfg},
like the text version of Score Guidance, applies an off-the-shelf model to the denoised image, a
blurred average, early in sampling; readout guidance \citep{luo2024readout} trains light heads on inner
features with task labels. Instead, our method is unsupervised and only applies an encoder on clean, non-blurry images.

The closest method to ours is that of unsupervised self-debiasing
\citep{korada2025unsupervised}, the only debiasing method that is given no attribute at all. For each prompt, it
trains a projector from the $h$-space of \citet{kwon2023asyrp}, to an embedding of the image that the sampler produces, then clusters the projections and steers each
batch towards balanced clusters. The attribute is therefore discovered rather than named, and a target
can only be set over the clusters found. Our adapter is instead trained once per generator rather than
per prompt, and is simulation-free: it regresses the embedding of a clean image from noised versions of
that image rather than from the sampling trajectory that produced it (\cref{sec:adapter}). Sentences then name the
attribute, and our batch assignment can set any proportion.

% ===========================================================================
\section{Preliminaries}
\label{sec:prelim}

\paragraph{Generative model.}
We focus in this paper on diffusion models, even though our method can easily be adapted to standard flow matching methods.
We follow the EDM parameterisation \citep{karras2022edm} without loss of generality; other diffusion models such as DDPM \citep{ho2020ddpm} and VP\,/\,VE SDEs \citep{song2021sde} may be written in the same framework.

A noisy image is
$x_\sigma = x_0 + \sigma\varepsilon \sim p_\sigma$ with $x_0 \sim p_{\mathrm{data}}$ and
$\varepsilon \sim \mathcal{N}(0,I)$, and a network $\Dth(x;\sigma)$ is pre-trained to approximate the
posterior mean $\mathbb{E}[x_0 \mid x_\sigma = x]$. We write $\xhat(x) = \Dth(x;\sigma)$ and call it
the \textit{Tweedie estimate}, the model's approximation (as an average) of the denoised image given the current noisy state. 
Tweedie's identity \citep{efron2011tweedie}, $\nabla_x \log p_\sigma(x) = (\Dth(x;\sigma) - x)/\sigma^{2}$,
turns the denoiser into the score of $p_\sigma$,
% which the sampler integrates from $\sigma_{\max}$ down to zero: 
and sampling integrates from $\sigma_{\max}$  down to $\sigma = 0$, for any $\beta>0$:
\begin{equation}
    \label{eq:sde}
    \mathrm{d}x_\sigma = - (1 + \beta) \sigma \,\nabla \log p_\sigma(x_\sigma) \,\mathrm{d}\sigma + \sqrt{2 \beta \sigma} \,\mathrm{d}W_\sigma = (1 + \beta) \frac{x - \Dth(x_\sigma;\sigma)}{\sigma} \,\mathrm{d}\sigma + \sqrt{2 \beta \sigma} \,\mathrm{d}W_\sigma.
\end{equation}
Guidance may be used to steer the sampler towards images with a desired property without retraining the model.
One chooses a scalar function $f$ of the image and of the noise level, the \emph{guidance term}, and a weight $w$, and
samples from $p_\sigma(x)\,e^{w f(x; \sigma)}$ instead of $p_\sigma$. By Tweedie's identity this only changes
the denoiser \citep{dhariwal2021beatgans,chung2023dps}:
\begin{equation}
  \Dtil(x;\sigma) \;=\; \Dth(x;\sigma) \;+\; w\,\sigma^{2}\,\nabla_x f(x; \sigma).
  \label{eq:replacement}
\end{equation}
Classifier guidance \citep{dhariwal2021beatgans} is the case $f = \log p_\sigma(y \mid x)$ for a class $y$,
estimated with a classifier trained on noisy images, and classifier-free guidance \citep{ho2022cfg} obtains
that classifier from a conditional model by Bayes' rule. The solver is untouched, and several guidance terms add inside one gradient.

\paragraph{Image-text embeddings.}
We will use a frozen image-text encoder $\imgenc$\,/\,$\textenc$ mapping respectively images and
texts to a shared embedding space on the unit sphere $\mathbb{S}^{d-1}$, so that matching pairs are close.
Any choice of image-text encoder, such as CLIP \citep{radford2021clip} and its successors, is admissible.
In this paper, we choose SigLIP~2 \citep{zhai2023siglip,tschannen2025siglip2}.

\section{Method}
\label{sec:method}

We describe here our guidance method for debiasing a pre-trained diffusion model $\Dth$ of any kind w.r.t.\ any sensitive attribute, expressed in free text, with arbitrary proportions and preserved output quality \& diversity, while requiring no extra annotated data and no diffusion model fine-tuning.

For this, we combine two guidance terms,
one for fairness and one for diversity, and compute both from a single semantic representation of the
image being generated. The proposed method therefore contains three pieces (\cref{fig:overview}): an adapter that makes a clean and easy-to-use embedding space of a pre-trained image-text encoder available at every noise level of a frozen
generator (\cref{sec:adapter}), attribute guidance directions from pairs of prompts with the batch-level
fairness term (\cref{sec:fairterm}), and a diversity term defined in that space
(\cref{sec:rarity}).

\begin{figure}[t]\centering
\resizebox{\linewidth}{!}{%
\begin{tikzpicture}[every node/.append style={font=\footnotesize}, x=1cm, y=1cm]
  % columns: A = denoiser, adapter, encoder; B = embedding; C = update, terms, directions
  \def\cA{5.7} \def\cB{8.9} \def\cC{12.0}
  % rows: 0 = sampling, \rU = adapter, \rT = prompts
  \def\rU{1.85} \def\rT{3.5}
  % ---- row 0: noisy image, denoiser, update, guided image
  \node[inner sep=0] (xs) at (2.2,0) {\includegraphics[width=1.45cm]{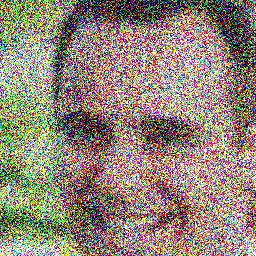}};
  \node[below=1pt of xs] {noisy image $x_\sigma$};
  \begin{scope}[shift={(\cA,0)}]
    \foreach \xx/\hh in {-1.53/0.78, -1.08/0.58, -0.63/0.40}
      \fill[DiagFrozen, draw=DiagFrozenL] (\xx,-\hh) rectangle (\xx+0.25,\hh);
    \fill[DiagEmb, draw=DiagEmbL, thick] (-0.175,-0.27) rectangle (0.175,0.27);
    \foreach \xx/\hh in {0.42/0.40, 0.87/0.58, 1.32/0.78}
      \fill[DiagFrozen!60, draw=DiagFrozenL, dashed] (\xx,-\hh) rectangle (\xx+0.25,\hh);
    \node[font=\scriptsize, DiagEmbL] at (0,-0.5) {$\hs(x_\sigma)$};
  \end{scope}
  \node[font=\scriptsize, inner sep=1pt] (unet) at (\cA,-1.05) {denoiser network $\Dth$ (frozen)};
  \draw[darr] (xs.east) -- (\cA-1.53,0);
  \node[frozen, align=center] (upd) at (\cC,0) {denoiser update\\$\Dtil = \Dth + \sigma^{2}\nabla_x[\,w_{\mathrm{fair}} f_{\mathrm{fair}} + w_{\mathrm{div}} f_{\mathrm{div}}\,]$};
  \draw[darr] (\cA+1.57,0) -- node[above, font=\scriptsize]{$\Dth(x_\sigma;\sigma)$} (upd.west);
  \node[inner sep=0, right=19mm of upd] (out) {\includegraphics[width=1.45cm]{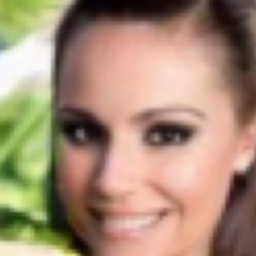}};
  \node[below=1pt of out] {guided image $x_0$};
  \draw[darr] (upd.east) -- node[above=1pt, font=\scriptsize, pos=0.6]{at $\sigma=0$} (out.west);
  % loop: the updated image is the noisy image of the next step
  \coordinate[right=6mm of upd.east] (fork);
  \fill[black!65] (fork) circle (1.4pt);
  \draw[darr] (fork) -- ++(0,-1.75) -- node[below, font=\scriptsize]{repeat at the next noise level until $\sigma=0$} (0.9,-1.75) -- (0.9,0) -- (xs.west);
  \node[ddata, dashed, font=\scriptsize, inner sep=2pt, align=center, anchor=west] (cond) at (\cA+2.4,-1.05) {condition or text prompt $c$\\from the end user (if any)};
  \draw[darr, dashed] (cond.west) -- (unet.east);
  % ---- row 1: adapter, embedding, guidance terms
  \node[trained, align=center] (ad) at (\cA,\rU) {adapter $\adapt_\psi$ (trained once)\\{\scriptsize inner layer $\hs$ $\to$ embedding space}};
  \draw[darr, DiagEmbL] (\cA,0.27) -- (ad.south);
  \node[demb] (z) at (\cB,\rU) {$m(x^{1:B};\sigma)$};
  \node[dbox, fill=DiagTrain, draw=DiagTrainL, align=center] (g) at (\cC,\rU) {fairness term $f_{\mathrm{fair}}(x^{1:B})$\\diversity term $f_{\mathrm{div}}$};
  \draw[darr] (ad) -- (z);
  \draw[darr] (z) -- (g);
  \draw[darr] (g.south) -- node[right=1pt, font=\scriptsize, align=left]{guidance} (upd.north);
  % ---- row 2: operator, prompts, image-text encoder, directions
  \begin{scope}[shift={(0,\rT)}]
    \fill[black!55] (0,0.32) circle (0.17);
    \fill[black!55] (-0.32,-0.3) .. controls (-0.32,0.12) and (0.32,0.12) .. (0.32,-0.3) -- cycle;
    \node[font=\scriptsize, below=8pt] at (0,-0.25) {operator};
  \end{scope}
  \node[ddata, align=left, font=\scriptsize] (prm) at (2.2,\rT) {$t_0$: \textit{a photo of a person}\\$t_1$: \textit{a photo of a man}\\$t_2$: \textit{a photo of a woman}\\targets $\alpha_1 = \alpha_2 = 0.5$};
  \node[frozen, align=center] (txt) at (\cA,\rT) {image--text encoder\\$\textenc$, $\imgenc$ (frozen)};
  \node[demb] (d) at (\cC,\rT) {directions $\Delta_k$};
  \draw[darr] (prm) -- (txt);
  \draw[darr] (txt) -- node[above, font=\scriptsize]{$\textenc(t_k)$} (d);
  \draw[darr, dashed] (txt.south) -- node[right=1pt, font=\scriptsize]{$\imgenc(x_0)$: training target} (ad.north);
  \draw[darr] (d.south) -- node[right=2pt, font=\scriptsize]{computed once} (g.north);
  % ---- inset, top right: directions and zero-shot classifier in the embedding space
  \begin{scope}[shift={(16.15,2.55)}]
    \draw[black!25, rounded corners=2.5pt] (-1.85,-1.3) rectangle (1.85,1.55);
    \node[font=\scriptsize, black!60, anchor=north east] at (1.85,1.55) {embedding space};
    % boundary between the two values: equal logits
    \draw[black!50, dashed] (0,-0.85) -- (0,0.65);
    \node[font=\scriptsize, black!60] at (0,-1.02) {$S_1 = S_2$};
    % directions from the neutral sentence
    \draw[-{Stealth[length=1.8mm]}, thick, DiagEmbL, shorten >=2.5pt, shorten <=2.5pt]
      (0,0.8) -- node[above right=-1pt, font=\scriptsize, DiagEmbL] {$\Delta_1$} (1.3,0.2);
    \draw[-{Stealth[length=1.8mm]}, thick, DiagEmbL, shorten >=2.5pt, shorten <=2.5pt]
      (0,0.8) -- node[above left=-1pt, font=\scriptsize, DiagEmbL] {$\Delta_2$} (-1.3,0.2);
    % clean images: men (triangles) and women (circles)
    \foreach \px/\py in {0.55/-0.35, 0.85/-0.15, 1.05/-0.5, 0.75/-0.65, 1.3/-0.25}
      \fill[DiagTrainL, shift={(\px,\py)}] (0,1.6pt) -- (-1.4pt,-0.8pt) -- (1.4pt,-0.8pt) -- cycle;
    \foreach \px/\py in {-0.55/-0.35, -0.85/-0.15, -1.05/-0.5, -0.75/-0.65, -1.3/-0.25}
      \fill[DiagFrozenL!90] (\px,\py) circle (1.2pt);
    \node[font=\scriptsize, DiagTrainL] at (1.25,-1.02) {$S_1 > S_2$};
    \node[font=\scriptsize, DiagFrozenL!80!black] at (-1.25,-1.02) {$S_2 > S_1$};
    % sentence embeddings
    \foreach \px/\py in {0/0.8, 1.3/0.2, -1.3/0.2}
      \fill[DiagDataL, shift={(\px,\py)}] (-2pt,-2pt) rectangle (2pt,2pt);
    \node[font=\scriptsize, DiagDataL!70!black, above=1pt] at (0,0.8) {$t_0$};
    \node[font=\scriptsize, DiagDataL!70!black, right=1pt] at (1.3,0.2) {$t_1$};
    \node[font=\scriptsize, DiagDataL!70!black, left=1pt] at (-1.3,0.2) {$t_2$};
  \end{scope}
  \draw[black!40, densely dotted] (d.east) -- (14.3,3.5);
\end{tikzpicture}}
\caption{One guided sampling step. The prompts $t_k$ are set once and
encoded into directions $\Delta_k$. At each step, the adapter $\adapt$ maps an inner layer $\hs(x_\sigma)$ of
the frozen denoiser to $m(x_\sigma;\sigma)$ for every image of the batch, where the fairness and
diversity terms are computed.}
\label{fig:overview}
\end{figure}
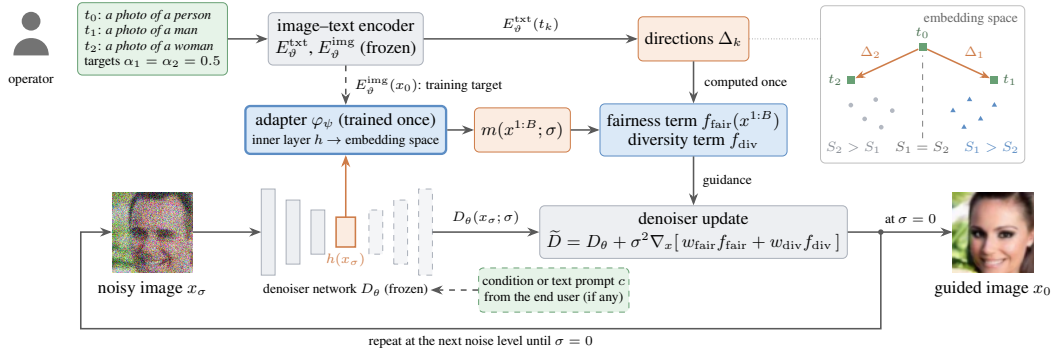

\subsection{A clean simulation-free semantic embedding of the noisy image}
\label{sec:adapter}

To build guidance terms without labeled data, we use the pre-trained image-text encoder $\imgenc$\,/\,$\textenc$, which turns any property described in words into a direction for guidance in image space.
% To construct our guidance terms without any supervision from labeled data, we leverage the provided pre-trained image-text encoder $\imgenc$\,/\,$\textenc$ which can transfer any property described in words into actionable directions for diffusion guidance in the image space. 
However, a guidance term $f$ needs to be evaluated on noisy images $x_\sigma$ (cf.\ \cref{sec:prelim}), on which the image-text encoder was never trained on. This raises the question of how to obtain the encoder's embedding $m(x_\sigma;\sigma)$ of the final image from the noisy one.

\paragraph{Choice of encoder input.} Two options have been considered in the literature.

The simplest option, used by training-free guidance
\citep{chung2023dps,yu2023freedom,bansal2024universal}, is to denoise $x_\sigma$ in one step to the
Tweedie estimate $\xhat \approx \mathbb{E}[x_0 \mid x_\sigma]$ and to embed that image to $m(x_\sigma;\sigma) = \imgenc(\xhat)$.
In this case, $\imgenc$ is applied to an average of images, necessarily blurry for high $\sigma$s and therefore out of distribution (OOD) for $\imgenc$, making it unreliable.
This distribution shift is particularly strong in the first phases of the trajectory when the coarse content of a sample, and with it an
attribute such as gender, is decided \citep{raya2023symmetry,biroli2024dynamical,gagneux2025phases}, i.e.\ when guidance needs to act.
We refer to \cref{app:potential} for more details.

Another option explored by \citet{korada2025unsupervised} is to embed the generated sample obtained after integrating \cref{eq:sde}.
To avoid the simulation cost at each guidance step, this necessitates the training of an adapter network trained to map $x_\sigma$ or some proxy like internal activations of $\Dth(x;\sigma)$ to the true embedding $m(x_\sigma;\sigma) = \imgenc(x_\sigma + \int_\sigma^0 \text{[\cref{eq:sde}]})$.
This solves the issue of blurry OOD inputs for the encoder, but ties guidance to the choice of generating process, including discretization errors and other guidance terms.

Instead, we propose to embed $x_0 = x_\sigma - \sigma \varepsilon$. When learning an adapter to regress this target, then:
\begin{equation}
    \label{eq:embedding}
    m(x_\sigma;\sigma) \;=\; \mathbb{E}[\imgenc(x_0) \mid x_\sigma].
\end{equation}
It combines the advantages of the two previous methods, being simulation-free and computed in-distribution, at the cost of averaging in the embedding space.
As shown in \cref{sec:fairterm}, this is a mild approximation as, due to its semantic nature, averaging in the embedding space better preserves attribute information than in the pixel space.

\paragraph{Our adapter.} To approximate $m$, we also use an adapter taking as input the denoiser
bottleneck activation $\hs(x_\sigma)$, which \citet{kwon2023asyrp} identified as a semantic latent
space (the $h$-space), and which the denoiser computes in any case during the generation process. In a U-Net
\citep{ronneberger2015unet}, $\hs$ can be the output of the mid-block; in a diffusion transformer
\citep{peebles2023dit} $\hs$ can then be the tokens of an intermediate block \citep{yu2025repa}.
The adapter network $\adapt_\psi$, with
parameters $\psi$, maps this activation to a unit vector $\adapt_\psi\big(\hs(x_\sigma);\,\sigma\big) \;\in\; \mathbb{S}^{d-1}$
and is trained with a cosine loss on pairs $(x_\sigma, \imgenc(x_0))$ at random noise levels, whose minimizer is the direction of $m$:
\begin{equation}
  \mathcal{L}(\psi) \;=\; 1 - \mathbb{E}_{x_0,\sigma,\varepsilon}\,\big\langle \adapt_\psi(\hs(x_\sigma);\sigma),\; \imgenc(x_0) \big\rangle,
  \label{eq:adapterloss}
\end{equation}
Only $\psi$ is trained; the denoiser and the encoder stay
frozen and no attribute label enters.
The adapter is trained once per generator and serves every attribute afterwards
(\cref{app:adapter}). 

\subsection{Fairness guidance from text prompts}
\label{sec:fairterm}
\label{sec:objective}

For a $K$-valued sensitive attribute, we are provided with a description of each value $k \in \llbracket 1, K \rrbracket$ by a sentence $t_k$, and a neutral sentence $t_0$: for age, for instance,
$t_0 = \text{\emph{a photo of a person}}$, $t_1 = \text{\emph{a photo of a young person}}$, $t_2 = \text{\emph{a photo of a middle-aged person}}$, $t_3 = \text{\emph{a photo of an old person}}$.
We will build a fairness guidance term using only these as information about the attribute.

\paragraph{Zero-shot classifier.} Thanks to the image-text encoder, a zero-shot classifier \citep{radford2021clip} can be built by comparing each attribute to the neutral prompt. Each value $k$
defines a direction $\Delta_k$ in the joint embedding space, and projecting an image embedding from \cref{sec:adapter} on $\Delta_k$ gives
a logit $S_k$ that increases the closer the image is to $t_k$ than to $t_0$ (\cref{fig:overview}, top right):
\begin{align}
  \Delta_k = \frac{\textenc(t_k) - \textenc(t_0)}{\lVert \textenc(t_k) - \textenc(t_0)\rVert},
  &&
  S_k(x_\sigma;\sigma) = \big\langle m(x_\sigma;\sigma),\, \Delta_k \big\rangle \in [-1,1].
  \label{eq:delta}
\end{align}
Directions are computed once before sampling. The difference of embeddings keeps what
distinguishes the two sentences and cancels what they share \citep{bolukbasi2016man,lyu2023deltaedit}.
In practice, we compute the logits with the adapter output
$\adapt_\psi(\hs(x_\sigma);\sigma)$ in place of $m(x_\sigma;\sigma)$. Note that the same adapter
serves every attribute.

At $\sigma = 0$, $m(x_0;0) = \imgenc(x_0)$ so it is the true embedding of the clean image. This defines a clean-image one-versus-rest classifier as $p(Y_k = 1 \mid x_0) \propto \mathrm{e}^{w S_k(x_0;0)}$, with $w > 0$ the guidance strength and $Y_k = \delta_{y=k}$. By linearity, $S_k(x_\sigma;\sigma) = \mathbb{E}[S_k(x_0;0) \mid x_\sigma]$: the logit of the noisy image
is the expected logit of the clean one. Our classifier thus
averages the log of the clean classifier:
\begin{equation}
  \bar p_\sigma(Y_k = 1 \mid x_\sigma) \;=\; \exp \mathbb{E}[\log p(Y_k = 1 \mid x_0) \mid x_\sigma] \;\propto\; \mathrm{e}^{w S_k(x_\sigma;\sigma)}.
  \label{eq:cls-ours}
\end{equation}
It is not the ideal such classifier, which is the average of the clean classifier (\cref{sec:prelim}):
\begin{equation}
  p_\sigma(Y_k = 1 \mid x_\sigma) \;=\; \mathbb{E}[p(Y_k = 1 \mid x_0) \mid x_\sigma] \;\propto\; \mathbb{E}[\mathrm{e}^{w S_k(x_0;0)} \mid x_\sigma],
  \label{eq:cls-ideal}
\end{equation}
but which is intractable in a zero-shot, attribute-independent setting. The gap is as follows.

\begin{lemma}[expected logit]\label{lem:score}
With $V_k(x) = \mathrm{Var}[S_k(x_0;0) \mid x_\sigma]$,
\begin{equation}
  e^{-2w}\,\tfrac{w^{2}}{2}\,V_k(x) \;\le\; \log p_\sigma(Y_k = 1 \mid x_\sigma) - \log \bar p_\sigma(Y_k = 1 \mid x_\sigma) \;\le\; e^{2w}\,\tfrac{w^{2}}{2}\,V_k(x).
  \label{eq:gap}
\end{equation}
\end{lemma}

The proof is in \cref{app:readouts}. The gap between the ideal classifier and ours therefore vanishes when $V_k$ approaches $0$, that is, once the model has decided the attribute.

\paragraph{Guidance.} Given target proportions $\alpha_1, \dots, \alpha_K \ge 0$ summing to one and a batch of $B$ images,
the objective is that a fraction $\alpha_k$ of the batch carries attribute value $k$. The objective is
enforced by batch-wise assignment: each image $i$ is assigned a value $\pi(i)$ and guided towards it
with our classifier of \cref{eq:cls-ours}. With integer quotas $B_k \approx \alpha_k B$ summing to $B$,
let $\Pi$ be the set of maps $\pi \colon \{1, \dots, B\} \to \{1, \dots, K\}$ that assign exactly
$B_k$ images to value $k$. The assignment, solved exactly by the Hungarian algorithm \citep{kuhn1955hungarian}, maximises the total logit of the batch, so that each image
is pushed, as far as the quotas allow, towards the value it is already closest to.
\begin{align}
  \pi^{\star} \;=\; \operatorname*{arg\,max}_{\pi \in \Pi} \sum_{i=1}^{B} a_{i\pi(i)},
  &&
  f_{\mathrm{fair}}(x^{1:B}) \;=\; \sum_{i=1}^{B} a_{i\pi^{\star}(i)}, && a_{ik} = S_k(x_i;\sigma).
  \label{eq:fair}
\end{align}
Up to a constant, $w\,f_{\mathrm{fair}} = \sum_i \log \bar p_\sigma(y = \pi^{\star}(i) \mid x_i)$ is the
log of our classifier summed over the batch, so guiding with $f_{\mathrm{fair}}$ is guiding each image
towards its assigned value.

\begin{proposition}\label{prop:quota}
Let $A = (a_{ik}) \in \mathbb{R}^{B \times K}$. Then
(i)~$f_{\mathrm{fair}} = \max_{\pi \in \Pi} \sum_{i} a_{i\pi(i)}$ is convex and piecewise linear in
$A$, and wherever the maximiser $\pi^{\star}$ is unique it is differentiable with
$\partial f_{\mathrm{fair}} / \partial a_{ik} = \mathbf{1}[\pi^{\star}(i) = k]$;
(ii)~the guided denoiser of \cref{eq:replacement} with $f = f_{\mathrm{fair}}$ is then:
\begin{equation}
  \Dtil(x_i;\sigma) \;=\; \Dth(x_i;\sigma) \;+\; w_{\mathrm{fair}}\,\sigma^{2}\, \nabla_{x_i} f_{\mathrm{fair}}(x^{1:B})
  \;=\; \Dth(x_i;\sigma) \;+\; w_{\mathrm{fair}}\,\sigma^{2}\, \nabla_{x_i} a_{i\pi^{\star}(i)} ,
  \label{eq:fairguidance}
\end{equation}
so each image is pushed towards the value it is assigned to.
\end{proposition}

The proof is in \cref{app:readouts}. 
By \cref{eq:fairguidance}, pushing each image towards its assigned value is guidance with $f_{\mathrm{fair}}$, and recomputing the assignment at every step recomputes its gradient.
Because the assignment couples the $B$
trajectories, \cref{eq:fair} is a joint guidance term over the batch in the sense of
\citet{corso2024particle}. Still, unlike \citet{parihar2024balancing}, where every image
receives the gradient of a single batch-level distance between the predicted and the target attribute
distributions, each image is guided by its own logit,
towards the value it is assigned to and therefore compatible with.

\paragraph{Binary vs multi-class.}

A multi-class zero-shot classifier on the same scores would lead to the same method with fewer guarantees: the assignment is unchanged, and the guidance term gains the normalisation
$-\log \sum_{j} e^{w S_j(x_0;0)}$; this term is not linear in the embedding, so \cref{lem:score} would
no longer hold exactly.
For a binary attribute,
the neutral sentence is not needed and the assignment may be reduced to ranking the batch along a single
direction (\cref{app:binary}).

\subsection{The diversity term and combined guidance}
\label{sec:rarity}

\paragraph{Diversity guidance.}
We start from SGMS~\citep{um2024sgms} and move its construction into the
encoder's space, available from the adapter at no extra cost.
SGMS scores diversity by re-noising the Tweedie estimate to $\sigma_s$, denoising it again, and
measuring how far the image moved:
\begin{equation}
  \hat x_s \;=\; \sg{\xhat(x)} + \sigma_s\varepsilon,
  \qquad
  f_{\mathrm{SGMS}}(x) \;=\; d_{\mathrm{pix}}\big(\xhat(x),\, \Dth(\hat x_s;\sigma_s)\big),
  \label{eq:sgms}
\end{equation}
with $d_{\mathrm{pix}}$ an LPIPS \citep{zhang2018lpips} or $\ell^{2}$ distance and $\sg{\cdot}$ a
stop-gradient.

Re-denoising pulls an
image towards the modes of the model, so an image already on a mode barely moves while an image in
the tail does. Measured in pixel space, however, the displacement counts pose, background and
lighting rather than semantics. We propose instead a
score that measures the disagreement between two estimates of the embedding of the image being
generated:
\begin{equation}
  f_{\mathrm{div}}(x) \;=\; 1 - \big\langle m(x_\sigma;\sigma),\; m(\hat x_s;\sigma_s) \big\rangle \;\approx\; 1 - \big\langle \adapt_\psi(\hs(x_\sigma);\sigma),\; \adapt_\psi(\hs(\hat x_s);\sigma_s) \big\rangle .
  \label{eq:perturbproj}
\end{equation}
\cref{app:siglipms}
describes a variant that takes the second estimate from the encoder instead.

\paragraph{Combined guidance.}\label{sec:combined}
The diversity term is added to the fairness guidance of \cref{eq:fairguidance}:
\begin{equation}
  \Dtil(x;\sigma) \;=\; \Dth(x;\sigma) \;+\; \sigma^{2}\,\nabla_x\!\left[\, w_{\mathrm{fair}}\,f_{\mathrm{fair}}(x) \;+\; w_{\mathrm{div}}\,f_{\mathrm{div}}(x) \,\right].
  \label{eq:combined}
\end{equation}

\paragraph{Compatibility with any conditioning and other guidance}
\label{sec:t2i}

The presented method is independent of how the model is conditioned, applies to
unconditional and text-conditional models alike, and composes with any other guidance already in use,
since~\cref{eq:combined} would simply gain a term. When classifier-free guidance \citep{ho2022cfg} is used, the adapter only reads from prompt-conditioned branch.

\begin{figure}[t]\centering
\rowfig{0.8\linewidth}{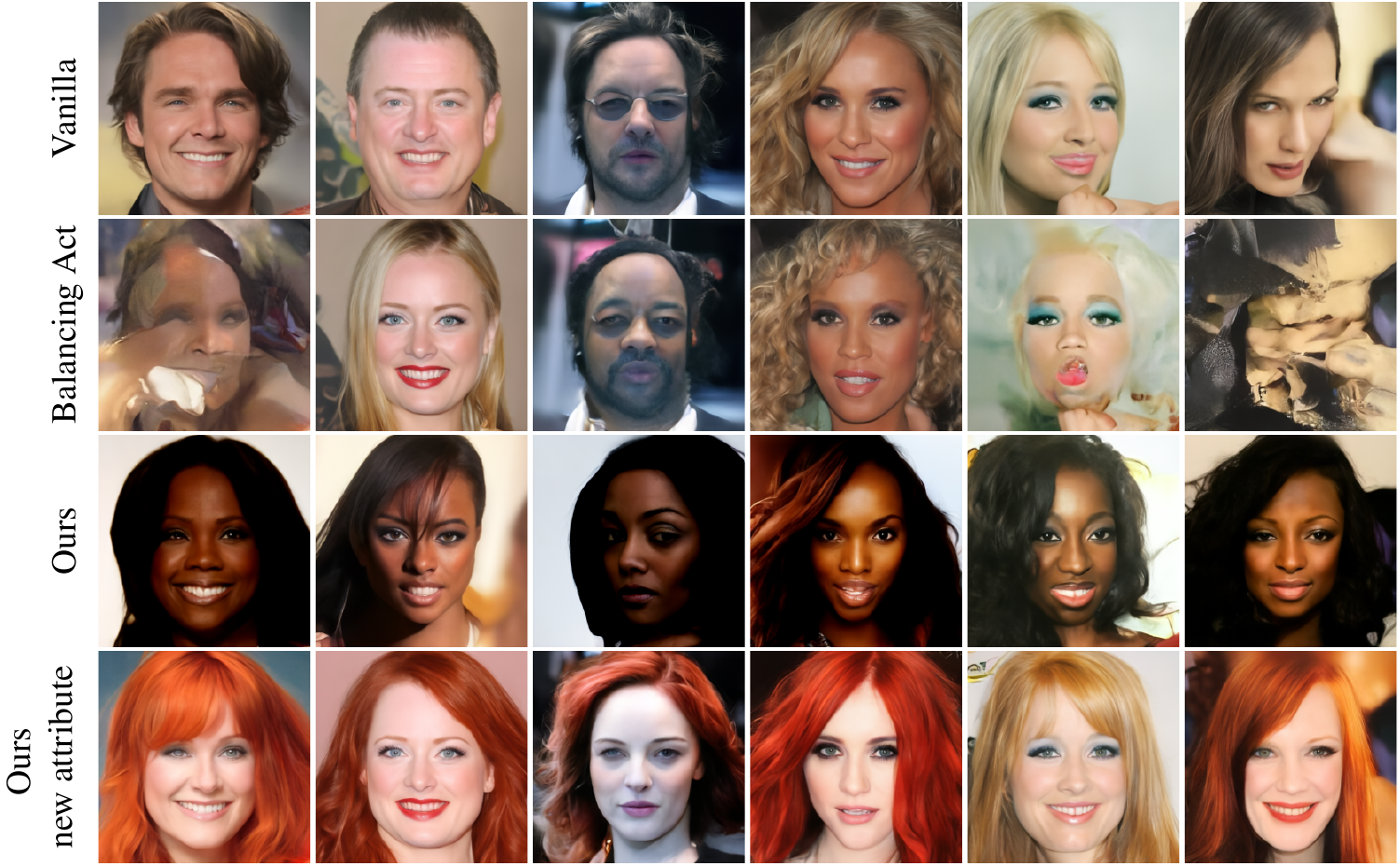}{51}{0.874/{Original}, 0.625/{Balancing Act \\ \textit{dark-skinned w.}}, 0.375/{\method{} \\ \textit{dark-skinned w.}}, 0.126/{Debias Anything \\ \textit{red-hair w.}}}
\caption{CelebA-HQ, P2 backbone, six shared seeds per column. Top: unguided samples. Second row:
Balancing Act steered towards a dark-skinned woman with its labelled predictor. Third row:
\method{} on the same target, given as two sentences. Bottom: \method{} on a second attribute
with no CelebA-HQ label, red hair, from the same seeds. 
}
\label{fig:celebahq-qual}
\end{figure}

% ===========================================================================
\section{Experiments}
\label{sec:results}

\label{sec:experiments}

We  compare in this section the proposed method, \method{}, with published, more restrictive (\cref{tab:positioning}), baselines at matched
fairness, on unconditional and a text-to-image models. 

\paragraph{Setup.}
\label{sec:res-setup}
We test \method{} on three generators: an EDM \citep{karras2022edm} trained on CelebA
\citep{liu2015celeba} at $64\times64$ (deferred to~\cref{app:celeba}), the P2 model of \citet{choi2022p2} trained on CelebA-HQ \citep{karras2018progressive} at
$256\times256$, and Stable Diffusion~1.5 \citep{rombach2022ldm}. On CelebA the attributes are: \emph{gender}
and \emph{eyeglasses}, binary with a target of one half; on CelebA-HQ: \emph{gender}, \emph{race} and \emph{eyeglasses}; on Stable
Diffusion: \emph{gender}, \emph{age} and \emph{race}, on occupation prompts that do not mention these attributes. The target proportion is uniform for each of these attribute possible values.
Fairness is measured
by the discrepancy $\FD$ between the achieved and target proportions, CLEAM-corrected where possible
\citep{teo2023measuring}.
Quality is assessed w.r.t.\ a debiased reference set with $\FID$ \citep{heusel2017fid} (and also $\MIND$ from \citet{berthet2026mind} in the appendix), by precision and recall \citep{kynkaanniemi2019precision},
and diversity with the Vendi score \citep{friedman2023vendi}, a diversity measure counting the effective number of distinct
images without a reference set. 
CelebA-HQ and Stable Diffusion comparisons follow the protocols of Balancing Act
\citep{parihar2024balancing} and DiffLens \citep{shi2025difflens}, so that their published numbers
can be reused. Prompts, baselines, metric definitions, weights, and other details are included in
\cref{app:details}.

\subsection{Comparison with baselines}
\label{sec:res-baselines}

\paragraph{CelebA-HQ.} Every number in \cref{tab:celebahq} comes from our
own runs of the three methods under the protocol of Balancing Act~\citep{parihar2024balancing}. For Score Guidance~\citep{tiwary2026scoreguidance}, the numbers we obtain differ from those reported in the original paper. We reproduced its
attribute classifier, and we believe the gap comes from the reference set behind $\FID$, which
might not be the same as Balancing Act; the Score Guidance version we run is
the stronger of the two suggested in the paper. \cref{tab:celebahq} reports our method at two weights, \textit{fairness setting} and \textit{quality setting}, since fairness guidance might reduce generation quality.

At the \textit{fairness setting} weight the discrepancy is the lowest in race and eyeglasses and, in
gender, second best behind \textbf{Score Guidance, which compensates for its low discrepancy with worse $\FID$}: at similar discrepancy, our $\FID$ is $7$ to $14$
points below the best version of the concurrent method. At the \textit{quality setting} weight, we set the discrepancy to approach Balancing Act's by tweaking $w_{\mathrm{fair}}$, so that only $\FID$ differs between the two bracketed rows: it is five points
lower for us on eyeglasses, and $2.2$ and $0.6$ points higher on race and gender. Adding
the diversity term at the same weight lowers discrepancy and $\FID$ together on gender and race, and takes $\FID$ below Balancing Act on gender and eyeglasses. So the diversity term
closes the $\FID$ gap that the fairness term leaves, still without any label. 
\cref{fig:celebahq-qual} shows additional advantage of the proposed method: the same model steered towards a target \textit{red-hair woman} that CelebA-HQ does not have a label for. Thus, \textbf{Debias Anything generates the image proportion closer to the target than a classifier trained on labeled images}, and with the diversity term \textbf{it gives a lower or comparable $\FID$ at comparable discrepancy.}

\begin{table}[t]\centering
\setlength{\tabcolsep}{5pt}
\caption{CelebA-HQ, P2 backbone, protocol of Balancing Act ($10\,000$ images, one seed); all rows
are our runs. $\FD$ is CLEAM-corrected except on race, which has no label to correct against. \Legendbest; the bracket marks the rows matched on
discrepancy.}
\label{tab:celebahq}
\vspace{\baselineskip}
\scriptsize
\begin{tabular}{lrrrrrr}
\toprule
& \multicolumn{2}{c}{Gender} & \multicolumn{2}{c}{Race} & \multicolumn{2}{c}{Eyeglasses} \\
\cmidrule(lr){2-3}\cmidrule(lr){4-5}\cmidrule(lr){6-7}
Method & \lo{$\FD$} & \lo{$\FID$} & \lo{$\FD$} & \lo{$\FID$} & \lo{$\FD$} & \lo{$\FID$} \\
\midrule
Original & 0.257 & 54.9 & 0.606 & 62.8 & 0.666 & 73.9 \\
Score Guidance \citep{tiwary2026scoreguidance} & \best{0.001} & \shade{15}{64.7} & \shade{98}{0.012} & \shade{15}{65.4} & \shade{93}{0.075} & \shade{15}{66.8} \\
\rowmark{mA}Balancing Act \citep{parihar2024balancing} & \shade{36}{0.043} & \shade{93}{50.1} & \shade{15}{0.182} & \best{52.0} & \shade{15}{0.187} & \shade{79}{50.5} \\
\midrule
\rowmark{mB}\method{}, quality setting & \shade{15}{0.057} & \shade{89}{50.7} & \shade{26}{0.159} & \shade{86}{54.2} & \shade{18}{0.183} & \shade{99}{45.4} \\
\quad $+$ diversity term & \shade{26}{0.050} & \best{48.7} & \shade{32}{0.147} & \shade{96}{52.7} & \shade{17}{0.184} & \best{45.2} \\
\method{}, fairness setting & \shade{89}{0.008} & \shade{89}{50.7} & \best{0.008} & \shade{61}{58.2} & \best{0.065} & \shade{43}{59.8} \\
\quad $+$ diversity term & \shade{89}{0.008} & \shade{93}{50.0} & \best{0.008} & \shade{62}{58.0} & \best{0.065} & \shade{43}{59.8} \\
\bottomrule
\end{tabular}
\rowbracket{mA}{mB}
\end{table}

\paragraph{Stable Diffusion.} The baseline rows of \cref{tab:t2i} are the numbers the DiffLens and ICM papers report under the setting of~\citet{shi2025difflens} and \citet{zaleska2026icm}. We re-ran the unguided model under the same setup and
evaluation and recovered the corresponding numbers reported by DiffLens, so we directly report all their published numbers. \textbf{Every baseline knows the attribute in advance}, through labels, a classifier
or a training run for that attribute. \textbf{Ours does not}: the adapter is
trained once, on prompts about jobs that say nothing about gender, age or race, and the attribute
enters only at sampling, as two sentences, with one weight for all three attributes. \cref{fig:sd-qual} shows doctors and firefighters before
and after guidance on shared seeds. 

With 
much less information about the attributes, our method brings the discrepancy from $0.56$--$0.75$ down to $0.20$
on all three attributes, and image quality stays close to the unguided model's: $\FID$ is second on
gender and age, behind DiffLens, and first on race. \cref{app:t2i-full} reports the detailed explanation on the residual discrepancy. Some methods go lower on one attribute, always
with attribute-specific training or supervision, and none holds across the three; only ours is below $0.21$ on all three. Where DiffLens is ahead, on gender and age, it is advantaged by attribute knowledge: labeled images locate the features that carry the attribute inside the network,
and it edits those features directly. 
Thus, in the setting of \citet{shi2025difflens}, \textbf{our unsupervised results are better
than or comparable to those of methods that leverage labeled data}.

\begin{table}[t]\centering
\setlength{\tabcolsep}{4pt}
\caption{Stable Diffusion~1.5, protocol of \citet{shi2025difflens}, whose reported numbers fill the
baseline rows bar ICM. All baselines incorporate attribute knowledge (labeled data, attribute-specific training);
ours only gets two sentences. \Legendbest.}
\label{tab:t2i}
\vspace{\baselineskip}
\scriptsize
\begin{tabular}{lrrrrrr}
\toprule
& \multicolumn{2}{c}{Gender} & \multicolumn{2}{c}{Age} & \multicolumn{2}{c}{Race} \\
\cmidrule(lr){2-3}\cmidrule(lr){4-5}\cmidrule(lr){6-7}
Method & \lo{$\FD$} & \lo{$\FID$} & \lo{$\FD$} & \lo{$\FID$} & \lo{$\FD$} & \lo{$\FID$} \\
\midrule
Original & 0.564 & 120.1 & 0.752 & 120.1 & 0.558 & 120.1 \\
\midrule
\emph{Attribute-specific} \\
Latent editing \citep{kwon2023asyrp} & \shade{15}{0.408} & \shade{15}{166.1} & \shade{29}{0.682} & \shade{15}{200.9} & \shade{29}{0.524} & \shade{29}{153.1} \\
$h$-distribution \citep{parihar2024balancing} & \shade{43}{0.222} & \shade{43}{151.7} & \shade{43}{0.506} & \shade{43}{147.7} & \shade{15}{0.544} & \shade{58}{126.9} \\
Latent direction \citep{li2024selfdiscovering} & \shade{29}{0.305} & \shade{58}{129.4} & \shade{86}{0.052} & \shade{72}{113.8} & \best{0.175} & \shade{43}{128.3} \\
Fine-tuning \citep{shen2024finetuning} & \shade{86}{0.050} & \shade{29}{161.5} & \shade{15}{0.746} & \shade{29}{161.5} & \shade{72}{0.198} & \shade{15}{161.5} \\
DiffLens \citep{shi2025difflens} & \best{0.046} & \best{112.8} & \best{0.049} & \best{99.2} & \shade{43}{0.401} & \shade{72}{119.9} \\
ICM \citep{zaleska2026icm} & \shade{72}{0.087} & \shade{72}{122.1} & \shade{72}{0.133} & \shade{58}{114.9} & \shade{58}{0.266} & \shade{86}{116.9} \\
\midrule
\emph{Attribute-unsupervised} \\
\method{} & \shade{58}{0.200} & \shade{86}{113.8} & \shade{58}{0.206} & \shade{86}{110.6} & \shade{86}{0.195} & \best{114.7} \\
\bottomrule
\end{tabular}
\end{table}

\subsection{Ablations on CelebA}
\label{sec:res-ablations}

Results are included in \cref{app:celeba}; selected findings are the following. \textbf{Reading the score
through the adapter is what makes the rare attribute reachable.} At equal weights on eyeglasses
it drives the discrepancy from $0.573$ to $0.32$ while lowering $\MIND$ from $20.51$ to $8.9$,
whereas the same score read from the encoder on the denoised image never passes $0.489$ and
raises $\MIND$ (\cref{tab:hspace-vs-pixel,fig:hspace-vs-tweedie}).
\textbf{Every baseline diversity term improves on the fairness
term alone, but ours lowers the discrepancy most} (\cref{tab:fairdiverse-5seeds}). On gender all
four lower $\MIND$ from $2.52$ to about $2.0$, and our term
lowers the discrepancy from $0.031$ to $0.006$; on eyeglasses they lower the discrepancy and raise
Vendi, while $\MIND$ moves within one standard deviation and $\FID$ rises by up to $1.2$.
\textbf{Our method is resilient to changing the wording of attribute reference sentences}, as shown in \cref{fig:celebahq-prompts}.

\section{Conclusion}

We introduce \method{}, a generative model guidance that sets the proportions of a sensitive attribute described in words, in
unconditional and text-to-image diffusion models, with no labels and no generator retraining. A generic adapter gives access to the embedding space of an image-text encoder; from it, a batch-level fairness term enforces the target proportions and a diversity term
counteracts the loss of diversity that guidance causes. Although it receives very little attribute information,
\method{} is better than or comparable to methods that rely on attribute supervision.

\textbf{Limitations.} The method has the limits of its parts. The reconstituted attribute relies on the image-text embedding, with the encoder's own biases. 
Guidance can only push towards images the model can produce: on Stable Diffusion, male nurses are rare and remain hard to reach. Steering one attribute moves others, as eyeglasses moved gender on CelebA, so attribute combinations should be considered.
The proportion is
enforced on a batch and not at sample level.
The two guidance weights have a narrow useful range that we
select by sweeps as detailed in the appendix. Our evidence is on human faces only, at uniform target proportions, with single seeds on
CelebA-HQ and Stable Diffusion.

\textbf{Future work.} We believe that the work has several future directions. The adapter itself could serve to
audit a model's biases at a fraction of the cost of generating and classifying its samples, or to
stop a biased trajectory early. The same recipe, one adapter from an inner layer to an embedding shared
with text, carries over to video and audio diffusion models, whose encoders exist. 
Finally, the proportions we enforce are measured by classifiers; putting the correction of their
errors inside the loop, rather than after it, would let the sampler aim at the true proportion
instead of the counted one.

% \subsection*{AI use statement}

% In this work, we used generative AI tools to help derive our theoretical analysis and as writing and coding assistants.

% We have
% not used generative AI assistants to generate synthetic data, to develop theoretical models or
% conceptual frameworks, to propose or refine hypotheses, to design or give feedback on the
% methodology or the experiments, to implement the method, or to interpret the results; translation,
% dataset cleaning and reformatting, and qualitative or thematic data analysis are not applicable to
% this work.

% We used generative AI tools for the following tasks: to
% draft parts of the paper and edit it for readability, to assist in theorem proving, to create and edit software code for
% evaluation and plotting, to produce the LaTeX of tables and figures, to identify and summarize
% relevant literature, and to format references. We have reviewed all AI-assisted work. Every
% AI-drafted sentence was checked by the authors against the tables and the cited papers; every
% number in the paper comes from our own runs or from the cited papers; AI-edited code was read and
% tested by the authors before use. We take responsibility for the final content of this work,
% including text, claims or artifacts produced with the aid of generative AI.

\subsection*{Ethics statement}

This work makes a generator represent groups of people in proportions an operator chooses, and
the mechanism that sets a proportion to one half can set it to zero. Any method that controls an
attribute offers the same lever, including the label-based methods we compare with. Furthermore, the method makes a chosen distribution
enforceable and says nothing about which distribution is fair.

The attributes we steer (gender, race, age, skin tone and
eyeglasses) are named in text and assessed by classifiers we trained and external tools like FairFace or FID,
that carry biases of their own and may impose fixed categories, binary gender and four races among some of
them. Our discrepancy numbers are only as good as the used classifiers.

\subsection*{Reproducibility statement}

\cref{app:details} lists the generators, attributes and prompts, the baselines and what each is
given, the metric definitions, and the protocol of every comparison. \cref{app:adapter} gives the
adapter's layers, loss, and training data, \cref{app:celeba} every ablation table with seeds and
standard deviations, and \cref{app:readouts} the proofs. Our code is available at \url{https://github.com/theaudaudiffret/Debias_anything}.

% It is important that the work published in ICLR is reproducible. Authors are
% strongly encouraged to include a paragraph-long Reproducibility Statement at the
% end of the main text (before references) to discuss the efforts that have been
% made to ensure reproducibility. This paragraph should not itself describe
% details needed for reproducing the results, but rather reference the parts of
% the main paper, appendix, and supplemental materials that will help with
% reproducibility. For example, for novel models or algorithms, a link to an
% anonymous downloadable source code can be submitted as supplementary materials;
% for theoretical results, clear explanations of any assumptions and a complete
% proof of the claims can be included in the appendix; for any datasets used in
% the experiments, a complete description of the data processing steps can be
% provided in the supplementary materials. Each of the above are examples of
% things that can be referenced in the reproducibility statement.


\subsubsection*{Acknowledgments}

This project was provided with computing AI and storage resources by GENCI at IDRIS thanks to the grant 2025-A0191015707 on the supercomputer Jean Zay's A100 partition.

\bibliography{refs}
\bibliographystyle{iclr2027_conference}

\appendix
\clearpage
\startcontents[appendix]
\titlecontents{section}[1.6em]{\addvspace{2pt}\small\bfseries}{\contentslabel{1.6em}}{\hspace*{-1.6em}}{\hfill\contentspage}
\titlecontents{subsection}[3.9em]{\small}{\contentslabel{2.3em}}{\hspace*{-2.3em}}{\titlerule*[0.6pc]{.}\contentspage}
\printcontents[appendix]{}{1}{\section*{Appendix contents}\setcounter{tocdepth}{2}}

% ===========================================================================
\section{Additional method details}
\label{sec:background}
\label{app:background}

\subsection{Proofs of theoretical results}
\label{app:readouts}

\begin{proof}[Proof of \cref{lem:score}]
Fix $k$, $\sigma$ and $x_\sigma$. Write $p(Y_k = 1 \mid x_0) = c\,e^{w S_k(x_0;0)}$, with $c$ independent
of $x_0$, and $S = S_k(x_0;0) \in [-1,1]$ for the logit of the clean image, a function of $x_0$ under
the posterior given $x_\sigma$. Let $\Lambda(t) = \log \mathbb{E}[e^{t S} \mid x_\sigma]$. Then
$\log p_\sigma(Y_k = 1 \mid x_\sigma) = \log c + \Lambda(w)$ by~\cref{eq:cls-ideal}. By linearity,
$\Lambda'(0) = \mathbb{E}[S \mid x_\sigma] = S_k(x_\sigma;\sigma)$, so
$\log \bar p_\sigma(Y_k = 1 \mid x_\sigma) = \log c + w\,S_k(x_\sigma;\sigma) = \log c + w\,\Lambda'(0)$
by~\cref{eq:delta} and~\cref{eq:cls-ours}, and the gap is $\Lambda(w) - \Lambda(0) - w\,\Lambda'(0)$.
Since $S$ is bounded, $\Lambda$ is smooth and $\Lambda''(t) = \mathrm{Var}_t[S]$, the variance under the
posterior tilted by $e^{t S}$. Taylor's formula with integral remainder gives:
\begin{equation}
    \Lambda(w) - \Lambda(0) - w\,\Lambda'(0) \;=\; \int_0^w (w - t)\,\mathrm{Var}_t[S]\,\mathrm{d}t .
\end{equation}
Since $S \in [-1,1]$, the density of the tilted posterior with respect to the posterior,
$e^{t S}/\mathbb{E}[e^{t S} \mid x_\sigma]$, lies in $[e^{-2t}, e^{2t}]$. Hence, for every
$b \in \mathbb{R}$, $\mathbb{E}_t[(S - b)^{2}]$ lies between $e^{-2t}$ and $e^{2t}$ times
$\mathbb{E}[(S - b)^{2} \mid x_\sigma]$; taking the minimum over $b$ gives
$e^{-2t}V_k(x) \le \mathrm{Var}_t[S] \le e^{2t}V_k(x)$. Bounding $e^{\pm 2t}$ by $e^{\pm 2w}$ on $[0,w]$
and using $\int_0^w (w - t)\,\mathrm{d}t = w^{2}/2$ gives~\cref{eq:gap}.
\end{proof}

\begin{proof}[Proof of \cref{prop:quota}]
(i)~$f_{\mathrm{fair}}$ is the maximum of the finitely many linear functions
$A \mapsto \sum_i a_{i\pi(i)}$, $\pi \in \Pi$, hence convex and piecewise linear. Where the maximiser
$\pi^{\star}$ is unique, the inequalities that make it the maximiser are strict and remain so on a
neighbourhood of $A$, where $f_{\mathrm{fair}}$ therefore coincides with $\sum_i a_{i\pi^{\star}(i)}$;
its partial derivatives are $\mathbf{1}[\pi^{\star}(i) = k]$. (ii)~Since $a_{ik}$ depends on $x_i$
only, $\nabla_{x_i} f_{\mathrm{fair}} = \sum_k \mathbf{1}[\pi^{\star}(i) = k]\,\nabla_{x_i} a_{ik}
= \nabla_{x_i} a_{i\pi^{\star}(i)}$ wherever (i) applies, and substituting
into~\cref{eq:replacement} gives~\cref{eq:fairguidance}.
\end{proof}

\subsection{Binary attributes}
\label{app:binary}

For a binary attribute, each value is described against the other instead of against a neutral
sentence. With $t_{+}$ and $t_{-}$ the two sentences, for gender \emph{a photo of a woman} and
\emph{a photo of a man},
\begin{equation}
  \Delta \;=\; \frac{\textenc(t_{+}) - \textenc(t_{-})}{\lVert \textenc(t_{+}) - \textenc(t_{-})\rVert},
  \qquad \Delta_1 = \Delta, \qquad \Delta_2 = -\Delta .
  \label{eq:delta-binary}
\end{equation}
The clean logit $S(x_0;0) = \langle \imgenc(x_0), \Delta\rangle$ is positive when the image is closer to $t_{+}$
than to $t_{-}$, and its sign is the zero-shot classifier of \cref{sec:prelim}. With
$a_i = S(x_i;\sigma)$, the noisy logit of~\cref{eq:delta} along $\Delta$, the scores
of~\cref{eq:fair} are $a_{i1} = a_i$ and
$a_{i2} = -a_i$. Writing $s_i = +1$ if $\pi(i) = 1$ and $s_i = -1$ otherwise, the total score of an
assignment is $\sum_i s_i a_i$, with exactly $B_{+} = B_1$ positive signs. Moving a $+1$ from a
smaller score to a larger one cannot decrease this sum, so it is maximised by giving the positive
signs to the $B_{+}$ largest scores (the rearrangement inequality). The assignment
of~\cref{eq:fair} thus reduces to ranking the batch along $\Delta$,
\begin{equation}
  s_i =
  \begin{cases}
    +1 & \text{if } a_i \text{ is among the } B_{+} \text{ largest scores of the batch,}\\
    -1 & \text{otherwise,}
  \end{cases}
  \qquad
  f_{\mathrm{fair}}(x^{1:B}) \;=\; \sum_{i=1}^{B} s_i\,a_i ,
  \label{eq:fair-binary}
\end{equation}
and~\cref{eq:fairguidance} becomes
$\Dtil(x_i;\sigma) = \Dth(x_i;\sigma) + w_{\mathrm{fair}}\,\sigma^{2}\, s_i\,\nabla_{x_i} a_i$: each
image is pushed so that its score moves towards the pole it is assigned to. The maximiser is unique
wherever the $B_{+}$-th and $(B_{+}+1)$-th largest scores differ. The CelebA experiments of
\cref{sec:results} use this case.

\subsection{Alternative diversity term}
\label{app:siglipms}

A second diversity term, which we name \code{SigLIPMS}, takes the second estimate of~\cref{eq:perturbproj} from the
frozen image encoder applied to the Tweedie estimate:
\begin{equation}
  f_{\mathrm{div}}(x) \;=\; 1 - \big\langle m(x_\sigma;\sigma),\; \sg{\imgenc(\xhat(x))} \big\rangle
  \;\approx\; 1 - \big\langle \adapt_\psi(\hs(x_\sigma);\sigma),\; \sg{\imgenc(\xhat(x))} \big\rangle .
  \label{eq:siglipms}
\end{equation}
The term uses the very discrepancy that motivated the adapter: $m(x_\sigma;\sigma) =
\mathbb{E}[\imgenc(x_0) \mid x_\sigma]$ takes the expectation in embedding space and
$\imgenc(\xhat(x))$ in pixel space; the two coincide when the posterior is concentrated on one image,
and they differ when it spreads over several. The term is thus small where the model already knows
what it is generating, on the modes, and large in the tail. The two terms are the two natural ways to obtain a
second estimate, one external and one internal, and they fail differently: \code{SigLIPMS} inherits
the encoder's behaviour on blurred images, while our original diversity term inherits the adapter's own errors.
\cref{app:celeba} compares them per attribute.

\subsection{Density baseline}
\label{app:densitybaselines}

The diversity terms of \cref{sec:rarity,app:siglipms} compare two estimates of the same
image. A second family, after \citet{sehwag2022lowdensity}, compares $m(x_\sigma;\sigma)$ with a fixed
reference set $\mathcal{R}$ of $10{,}000$ embeddings $\imgenc(x_0)$ of real images and measures rarity
relative to the data. We use the mean distance to the $\kappa$ nearest neighbours (\code{avgKNN}),
\begin{equation}
  f_{\mathrm{div}}(x) \;=\; \frac{1}{\kappa} \sum_{r \in \mathcal{N}_\kappa(x)} \big\lVert m(x_\sigma;\sigma) - r \big\rVert ,
  \label{eq:avgknn}
\end{equation}
with $\mathcal{N}_\kappa(x)$ the $\kappa$ nearest neighbours of $m(x_\sigma;\sigma)$ in $\mathcal{R}$,
computed, as in \cref{sec:fairterm}, with the adapter output in place of $m(x_\sigma;\sigma)$; both lie
on the unit sphere, so this is a function of the inner products
$\langle \adapt_\psi(\hs(x_\sigma);\sigma), r\rangle$, $r \in \mathcal{R}$. The two families measure
different quantities, rarity relative to the data against the posterior spread of the model. They are compared in \cref{app:celeba}. In every case, $w_{\mathrm{div}} > 0$ increases the term and
therefore pushes towards low-density regions.

\section{Experimental details}
\label{app:details}

This section contains all details needed to reproduce \cref{sec:experiments}, except for adapters which are described in
\cref{app:adapter}.

\paragraph{Implementation details.} We implemented our method using PyTorch \citep{Paszke2019} with the help of SciPy \citep{2020SciPy-NMeth} for the matching algorithm. Our code is available at \url{https://github.com/theaudaudiffret/Debias_anything}. Every run uses one NVIDIA A100 (40\,GB) GPU.

\subsection{Generators}
\label{app:details-backbones}
On CelebA \citep{liu2015celeba} we train an EDM model \citep{karras2022edm} at $64\times64$: a
U-Net \citep{ronneberger2015unet} of $19.4$\,M parameters (base width $96$, multipliers $(1,2,2,2)$,
two residual blocks per level in the encoder and three in the decoder, self-attention at $16^2$,
$8^2$ and in the middle block, dropout $0.1$), trained with the EDM loss and preconditioning on the
$162{,}770$ training images (centre crop $140$, no flip) with Adam \citep{kingma2017adammethodstochasticoptimization}, batch
$256$, EMA $0.9999$ and a learning rate of $2\cdot10^{-4}$ divided by $2$, $10$ and $10$ at plateaus
of the validation loss; the EMA weights of lowest validation loss are reached after $1{,}487$ epochs
($9.5\cdot10^5$ steps, $8.5$ days). On CelebA-HQ \citep{karras2018progressive} we use the released
P2 model \citep{choi2022p2}, an ADM U-Net \citep{dhariwal2021beatgans} of $93.6$\,M parameters at
$256\times256$ (width $128$, multipliers $(1,1,2,2,4,4)$, attention at $16^2$) with a linear $\beta$
schedule over $1{,}000$ steps. The text-to-image model is the released Stable Diffusion~1.5
from \citet{rombach2022ldm} ($860$\,M-parameter U-Net, $512\times512$).

\subsection{Attributes, prompts and sample sizes}
\label{app:details-prompts}
On CelebA \citep{liu2015celeba} the attributes are gender (\emph{a photo of a man} / \emph{a photo
of a woman}) and eyeglasses (\emph{a person with eyeglasses or sunglasses} / \emph{a person without
eyeglasses or sunglasses}). On CelebA-HQ \citep{karras2018progressive}, following Balancing Act
\citep{parihar2024balancing}, gender keeps the CelebA pair, race adds \emph{a photo of a black
person} / \emph{a photo of a white person}, and eyeglasses is \emph{a photo of a person with
eyeglasses} / \emph{a photo of a person without eyeglasses}. All these attributes are binary and
use the form of \cref{app:binary}. On Stable Diffusion \citep{rombach2022ldm}, following DiffLens
\citep{shi2025difflens}, images are generated from \emph{A face of a \{doctor, firefighter, nurse,
receptionist\}}. Gender is binary (\emph{a photo of a woman} / \emph{a photo of a man}); age and race
use the $K$-valued form of \cref{sec:fairterm} with $t_0 = $ \emph{a photo of a person}, and
$t_1, t_2, t_3 = $ \emph{a photo of a young / middle-aged / old person} for age, $t_1, \dots, t_4 =
$ \emph{a photo of a white / black / asian / indian person} for race, the four FairFace groups of
the DiffLens protocol. The assignment of \cref{eq:fair} is solved with the Hungarian algorithm
\citep{kuhn1955hungarian} on a cost matrix in which the column of value $k$ is repeated $B_k$ times,
so that the quotas are met exactly. The targets are uniform, $\alpha_k = 1/K$. Each configuration is
sampled with five seeds of $5{,}000$ images on CelebA, one seed of $10{,}000$ images per attribute on
CelebA-HQ, and one seed of $500$ images per prompt on Stable Diffusion.

\subsection{Sampling}
\label{app:details-sampling}
\textbf{CelebA}: deterministic Heun solver \citep{karras2022edm}, $100$ steps on $\sigma \in
[0.002, 80]$ ($\rho = 7$), EMA weights, guidance for $\sigma \le \sigma_{\max} = 10$ ($67$ steps),
batch size $B = 100$; SGMS \citep{um2024sgms} and our diversity term of \cref{eq:perturbproj} act at
every fifth network evaluation and re-noise at the fixed level $\sigma_s = 0.96$ of \cref{eq:sgms},
\code{SigLIPMS} (\cref{app:siglipms}) at every evaluation. \textbf{CelebA-HQ}: DDIM
\citep{song2021denoising}, $50$ steps, $\eta = 1$, guidance at every step, batch size $B = 50$; the
unguided model and Balancing Act \citep{parihar2024balancing} are sampled with the released code on
the same schedule with batch size $B=100$. \textbf{Stable Diffusion}: Euler with trailing timesteps,
$30$ steps, classifier-free guidance \citep{ho2022cfg} at scale $7.5$, guidance for $\sigma \in
[0.2, 4]$ ($21$ steps), batch size $B = 100$ images of a single prompt. A run takes
$18$--$25$\,min per $5{,}000$ images on CelebA, $2.1$--$2.9$\,h per $10{,}000$ on CelebA-HQ and
$37$\,min per $2{,}000$ on Stable Diffusion.

\subsection{Baselines}
\label{app:details-baselines}
On CelebA-HQ we rerun Balancing Act \citep{parihar2024balancing}, a predictor trained on labelled
attributes from the same U-Net activations, at its published setting (scale $1{,}500$, batch size
$100$), and the exemplar variant of Score Guidance \citep{tiwary2026scoreguidance}, which steers a
share of each batch towards the DDIM-inverted scores of eight exemplar images per class, with the
authors' code and gender preset on all three attributes. Every method is then evaluated in the same
way, under the protocol of Balancing Act: the same reference sets and our own attribute classifiers
(\cref{app:details-metrics}). On Stable Diffusion we evaluate with the released code of DiffLens
\citep{shi2025difflens} and reproduce its unguided row ($\FD$ $0.565$, $0.747$ and $0.560$ on
gender, age and race against $0.564$, $0.752$ and $0.558$ reported; $\FID$ $117.1$ against
$120.1$), so that our numbers are comparable with those it reports for latent editing
\citep{kwon2023asyrp}, $h$-space distribution guidance \citep{parihar2024balancing}, discovered
latent directions \citep{li2024selfdiscovering}, fine-tuning \citep{shen2024finetuning} and DiffLens
itself; ICM \citep{zaleska2026icm}, which locates the attribute with labeled probes, is added under
the same protocol. On CelebA our two diversity terms are compared with SGMS \citep{um2024sgms}, the
pixel-space score they transpose. 

For~\cref{fig:celebahq-qual}, Balancing Act is run in its per-sample mode, in which every image is guided towards a fixed attribute value, here the combination of its gender and race predictors set to female and Black, at the same scale as the single-attribute runs. Its batch-level distribution mode, which matches a target proportion over the batch, did not reach the compound target in our runs.

\subsection{Metrics}
\label{app:details-metrics}

\paragraph{Reference sets.}
On CelebA \citep{liu2015celeba} the reference has the size of the generated set, drawn for the
diversity experiments from the minority set of SGMS \citep{um2024sgms}, the $10{,}000$ training
images of highest $\avgknn$ \citep{sehwag2022lowdensity} (5 nearest neighbours, ResNet-50 features
\citep{he2015deepresiduallearningimage}), and for the fairness experiments with replacement from the training set by
a sampler balanced on the attribute. On CelebA-HQ \citep{karras2018progressive}, as in Balancing
Act \citep{parihar2024balancing}, it is $5{,}000$ images balanced with the attribute labels for
gender and eyeglasses ($2{,}500$ per class; the $1{,}468$ images with eyeglasses are drawn with
replacement), and for race, which has no label, the $2{,}500$ images most and the $2{,}500$ least
similar to \emph{a black person} under CLIP ViT-B/32 \citep{radford2021clip}. On Stable Diffusion
we follow DiffLens \citep{shi2025difflens}, which takes FFHQ \citep{karras2019stylebasedgeneratorarchitecturegenerative} as the real
reference for text-to-image models and computes $\FID$ on the $2{,}000$ images of the four prompts
at once.

\paragraph{Quality and diversity.}
$\FID$ \citep{heusel2017fid} is computed on Inception-V3 features \citep{szegedy2015rethinkinginceptionarchitecturecomputer} on
CelebA and CelebA-HQ, and with Clean-FID \citep{parmar2022cleanfid} on Stable Diffusion. On CelebA
we add $\MIND$ \citep{berthet2026mind}, an optimal-transport distance on the same features
($1{,}000$ random projections, equal set sizes) that drops the Gaussian assumption of $\FID$, known
to misrank diffusion models \citep{stein2023exposing} and to bias the estimate with sample size
\citep{chong2020unbiasedfid}; precision and recall\citep{kynkaanniemi2019precision}, all with five nearest neighbours, and
density and coverage \citep{naeem2020density}, which separate fidelity from coverage; and the Vendi
score \citep{friedman2023vendi} with a cosine kernel, which needs no reference. On Stable Diffusion,
CLIP-T and CLIP-I \citep{radford2021clip,shi2025difflens} are the similarity of a sample to its
prompt and to the unguided sample of the same prompt; like $\FD$, they are computed per occupation
and averaged over the four.

\paragraph{Fairness.}
On CelebA and CelebA-HQ, $\FD$ is computed from hard predictions:
$\FD = \lVert (p, 1-p) - (\tfrac12, \tfrac12)\rVert_2 = \sqrt{2}\,\lvert p - \tfrac12\rvert$, with $p$
the proportion of class $0$ among hard predictions. For gender and eyeglasses, $p$ is the
CLEAM-corrected proportion $\pstar$ \citep{teo2023measuring} of class $0$ (female, no eyeglasses).
On CelebA the attribute is counted by two CNNs trained on the training split
with class-balanced sampling, at $98.8\%$ (gender) and $99.5\%$ (eyeglasses) validation accuracy,
with per-class accuracies $(\alpha_0, \alpha_1) = (0.987, 0.989)$ and $(0.997, 0.981)$. On
CelebA-HQ, as in Balancing Act \citep{parihar2024balancing}, gender and eyeglasses are counted by
ResNet-18s \citep{he2015deepresiduallearningimage} pretrained on ImageNet and fine-tuned on $80\%$
of CelebA-HQ ($97.9\%$ and $98.2\%$ balanced test accuracy), with per-class recalls
$(0.989, 0.968)$ and $(0.998, 0.966)$. Race, which has no label, is counted by FairFace
\citep{karkkainen2021fairface} restricted to its White and Black outputs; since the lack of labels
prevents a proper evaluation of this classifier, race is not CLEAM-corrected and $p$ is the raw
proportion of White predictions. On Stable Diffusion, as in DiffLens \citep{shi2025difflens}, faces
are cropped with dlib \citep{king2009dlib} and classified by FairFace \citep{karkkainen2021fairface}
into two genders, three ages (0--19, 20--59, 60+) and four races (White, Black, East and Southeast
Asian, Indian, not renormalised), and $\FD$ is the distance of the mean softmax output to the
uniform target, without CLEAM correction, averaged over the four occupations.

\subsection{Protocols}
\label{app:details-protocol}
On CelebA we keep the $w_{\mathrm{fair}}$ of lowest $\MIND$ in the sweep of \cref{app:sweeps} ($5$
on gender, $200$ on eyeglasses) and, for each diversity term, the $w_{\mathrm{div}}$ of best diversity
in the ablation of \cref{app:divsweeps}; entries are means over five seeds, with standard deviations
in \cref{app:fairdiverse}. On CelebA-HQ the fairness setting of \cref{tab:celebahq} uses
$w_{\mathrm{fair}} = 20$ / $100$ / $500$ on gender / race / eyeglasses, and the quality setting
$5$ / $30$ / $180$. With the diversity term in the fairness setting, we choose $w_{\mathrm{div}} = 40$, re-noising at $\sigma_s = 0.5\,\sigma$ every second step for $t \le 800$,
with the race weight lowered to $75$.
On Stable Diffusion $w_{\mathrm{fair}} = 3{,}000$ on every attribute.

\section{Adapter architecture and training}
\label{app:adapter}

\begin{table}[t]
\centering
\footnotesize
\setlength{\tabcolsep}{4pt}
\caption{Adapter training, on one A100 GPU.}
\vspace{\baselineskip}
\label{tab:training}
\begin{tabular}{@{}l>{\raggedright\arraybackslash}p{0.19\linewidth}>{\raggedright\arraybackslash}p{0.19\linewidth}>{\raggedright\arraybackslash}p{0.19\linewidth}rr@{}}
\toprule
Backbone & Data & Noise levels & Optimisation & Epochs & Time \\
\midrule
EDM, CelebA & $162{,}770$ training images & $\sigma$ log-uniform on $[0.002, 80]$
  & AdamW, batch 512, lr $10^{-3}$ & $50$ & $2.9$ h \\
P2, CelebA-HQ & $28{,}000$ ($+2{,}000$ val.) & $t$ uniform on the 49 guided steps
  & AdamW, batch 32, lr $10^{-3}$, wd $10^{-4}$ & $47$ of $60$ & $4.2$ h \\
SD~1.5 & $49{,}000$ generated ($+1{,}000$ val.) & $\sigma$ uniform on $[0.03, 14.6]$
  & AdamW, batch 64, lr $5\cdot10^{-4}$, wd $0.05$, warm-up and cosine & $47$ of $50$ & $11$ h \\
\bottomrule
\end{tabular}
\end{table}

The adapter $\adapt_\psi$ of \cref{sec:adapter} is learned from clean images only. Each training
image $x_0$ is noised at two independent levels, $x_\sigma = x_0 + \sigma\varepsilon$ (on P2
\citep{choi2022p2}, $x_t = \sqrt{\bar\alpha_t}\,x_0 + \sqrt{1-\bar\alpha_t}\,\varepsilon$), passed
through the frozen backbone, and $\adapt_\psi(\hs(x_\sigma);\sigma)$ is regressed onto the image
embedding $\imgenc(x_0)$ of the clean image, given by SigLIP~2 ViT-B/16 \citep{tschannen2025siglip2}
($768$ dimensions), under the cosine loss of \cref{eq:adapterloss}. No sampler trajectory is used. The
targets are fixed, one per image and precomputed once, so a contrastive loss would have nothing to
discriminate that this regression does not already do.

\paragraph{Unconditional backbones.}
On EDM \citep{karras2022edm} and P2 \citep{choi2022p2}, $\hs$ is the output of the U-Net
\citep{ronneberger2015unet} middle block ($192$ and $512$ channels at $8\times8$), read as $64$
tokens with a learned position embedding, pooled by attention with a learned query and a residual
MLP, normalised, modulated by FiLM \citep{perez2017filmvisualreasoninggeneral} on a sinusoidal embedding of the noise level
($\tfrac14\ln\sigma$ on EDM, $t/999$ on P2) initialised at zero, and mapped to $768$ dimensions by a
two-layer MLP followed by $\ell^2$ normalisation ($0.79$\,M and $2.42$\,M parameters).

\paragraph{Text-conditional backbone.}
The adapter reads the prompt-conditioned branch $\Dth(x;\sigma,c)$ of classifier-free guidance
\citep{ho2022cfg}:
\begin{equation}
\Dth^{\gamma}(x;\sigma) \;=\; \Dth(x;\sigma,\varnothing) \;+\; \gamma\,\big(\Dth(x;\sigma,c) - \Dth(x;\sigma,\varnothing)\big),
\label{eq:cfg}
\end{equation}
with $c$ the prompt, $\varnothing$ the empty prompt and $\gamma$ the guidance scale. Its input
combines the mid-block and the four decoder blocks of that branch:
\begin{equation}
\hs(x_\sigma) \;=\; \textstyle\sum_{\ell=1}^{5} \beta_\ell\, \pi_\ell\big(\hs_\ell(x_\sigma)\big),
\qquad \beta = \mathrm{softmax}(b),
\label{eq:dual}
\end{equation}
where $\hs_\ell$ are the block activations ($1280\times8^2$ to $320\times64^2$), $\pi_\ell$
timestep-conditioned bottlenecks to $512\times8\times8$ and $b$ learned logits; a six-layer
transformer (width $512$) with attention pooling maps the $64$ tokens to the embedding space
($21.2$\,M parameters in all). Since the logits are read before the combination \cref{eq:cfg}, they
do not depend on $\gamma$. A readout of the middle block alone reaches the gender target as well, but
not eyeglasses, which it tends to paint as a bright mask on the lower face.

\paragraph{Data for text-to-image.}
The adapter must read the $h$-space of this model, so its training images are generated by Stable
Diffusion~1.5 \citep{rombach2022ldm} itself: $50{,}000$ images at $512\times512$ (Euler with trailing
timesteps, $30$ steps, classifier-free guidance \citep{ho2022cfg} at scale $7.5$), in $6.2$\,h. Each
sample stores the clean latent $x_0$, its prompt and the target $\imgenc$ of the decoded image; at
every training step $x_0$ is noised and passed through the frozen U-Net. Four templates (\emph{a photo of}, \emph{a portrait photo of}, \emph{a realistic centered portrait
of}, \emph{a face of}) crossed with \emph{a person} and the four evaluation occupations are used ($20$
prompts). None name any attribute, so the attribute in each image comes from the model's own bias,
not from the text conditioning.

\paragraph{Training.}
\Cref{tab:training} lists the data, noise levels and cost of each adapter. The noise levels cover
the range where guidance acts: log-uniform over the full EDM range on CelebA, the $49$ guided DDIM
steps on P2, and uniform over the Stable Diffusion schedule \citep{rombach2022ldm}. On CelebA-HQ and
Stable Diffusion the kept epoch is the one of best validation retrieval (R@1 among the validation
targets). \Cref{fig:projection} shows what the CelebA adapter retains of the attribute as the noise
grows.

\begin{figure}[t]\centering
\includegraphics[width=\linewidth]{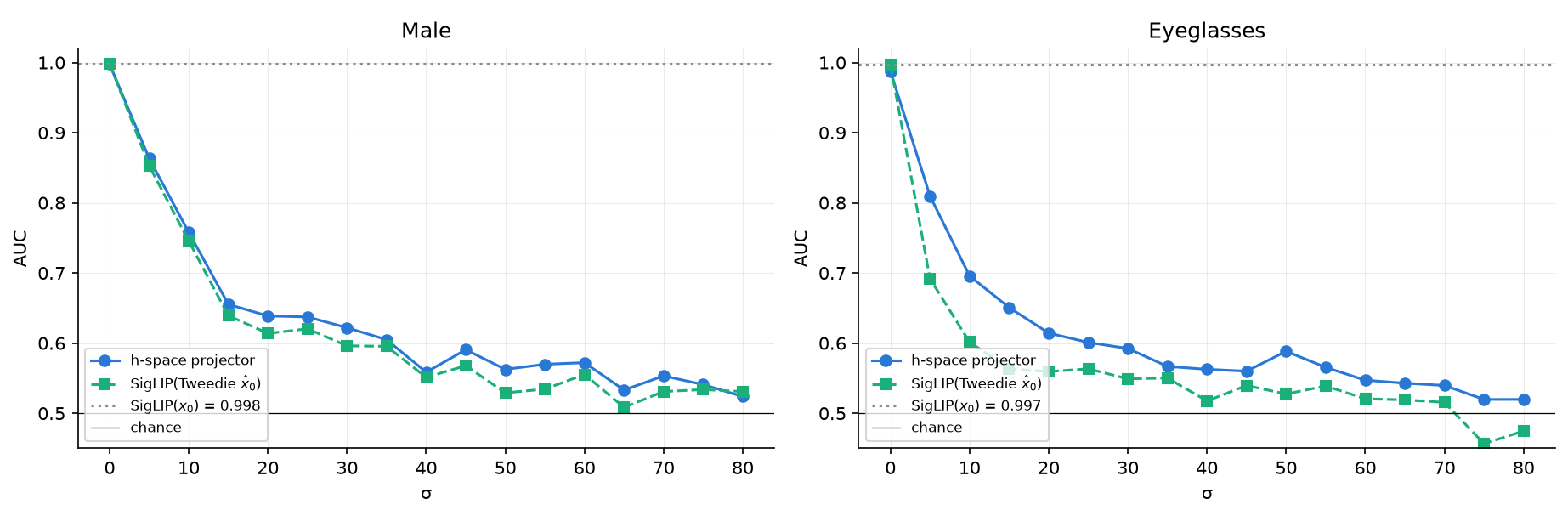}
\caption{What the CelebA adapter reads of the attribute as the noise grows. AUC of the logit
$S(x_\sigma;\sigma)$ along the direction $\Delta$ of \cref{eq:delta-binary}, against the CelebA
label, on $500$ images per class of the test split, for gender (\emph{a photo of a man} / \emph{a
photo of a woman}, left) and eyeglasses (\emph{a photo of a person wearing eyeglasses} / \emph{a
photo of a person}, right). Circles use the adapter output $\adapt_\psi(\hs(x_\sigma);\sigma)$,
squares the encoder on the Tweedie estimate, $\imgenc(\xhat)$; the dotted line is the encoder on the
clean image, $\imgenc(x_0)$. Over the guidance window $\sigma \le 10$ the adapter matches the Tweedie
estimate on gender and exceeds it on eyeglasses, at no decoding cost.}
\label{fig:projection}
\end{figure}

\section{Ablations on CelebA}
\label{app:celeba}

We test the following design choices: where the score
is read, the fairness weight, the diversity weight, and the choice of diversity term; a last
section runs each diversity term on its own, without the fairness term. The chosen diversity term is then used in CelebA-HQ and Stable Diffusion experiments.

The following properties are shared between all runs: the EDM backbone of \cref{app:details}, the deterministic Heun sampler, guidance
active for $\sigma \le \sigma_{\max} = 10$, and the fairness term of \cref{sec:fairterm} read in
the $h$-space through the adapter. Unless a section says otherwise, $\FID$, $\MIND$, precision, recall, density and coverage are
computed against real CelebA images resampled to be balanced on the attribute being steered, as
many as the generated images, and every $\FD$ is CLEAM-corrected. In every table each column is
shaded from its best value (dark, bold) to its worst (light), within each block where a table has
several; in the sweep tables the shading orders the settings of one method.

\subsection{Where the score is computed: $h$-space against pixel space}
\label{app:potential}

``Pixel'' here denotes the fairness term computed from SigLIP applied to the Tweedie estimate
$\xhat$, the readout that \cref{sec:adapter} argued moves the expectation through a nonlinearity
evaluated off the support of clean images, whereas the $h$-space estimate moves it through a linear
functional. \cref{fig:hspace-vs-tweedie} is the visible
counterpart on the compound target of \cref{fig:qual-compound}, and
\cref{tab:hspace-vs-pixel} and \cref{fig:hspace-vs-pixel} the quantitative one on the two
binary attributes, at equal guidance weights.

\begin{figure}[t]\centering
\rowfig{0.94\linewidth}{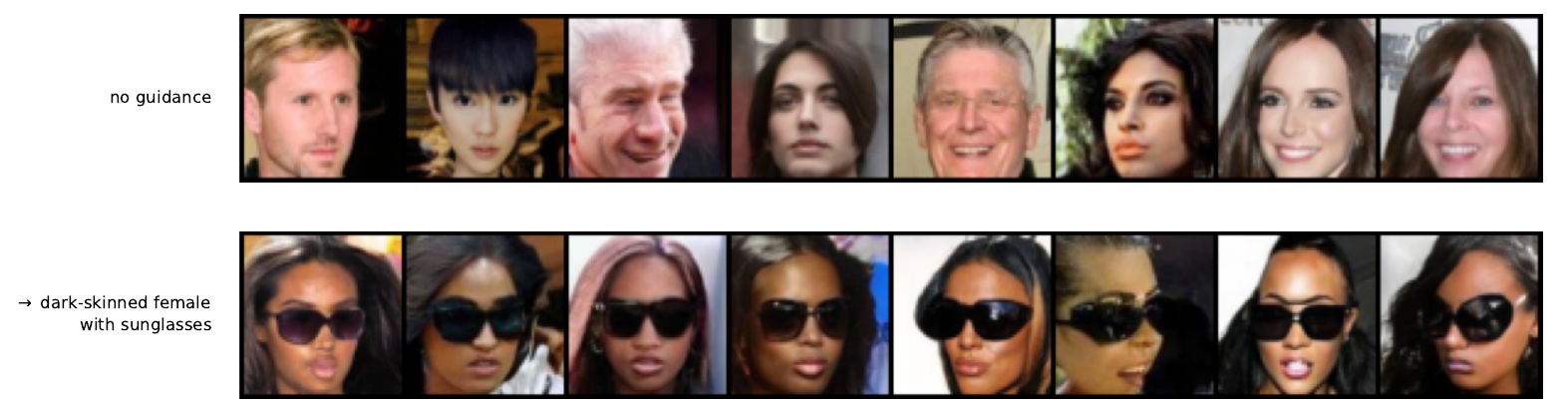}{110}{0.761/{Original}, 0.238/{dark-skinned female\\with sunglasses}}
\caption{Steering an attribute that has no label. Top: unguided samples from a diffusion model
trained on CelebA. Bottom: the same eight initial noises, guided towards \emph{dark-skinned female
with sunglasses}. The attribute is handed to the sampler as this sentence and a contrasting one
describing its absence. CelebA has no label for it, and neither the generator nor any classifier was
trained.}
\label{fig:qual-compound}
\end{figure}

\begin{figure}[t]\centering
\rowfig{0.94\linewidth}{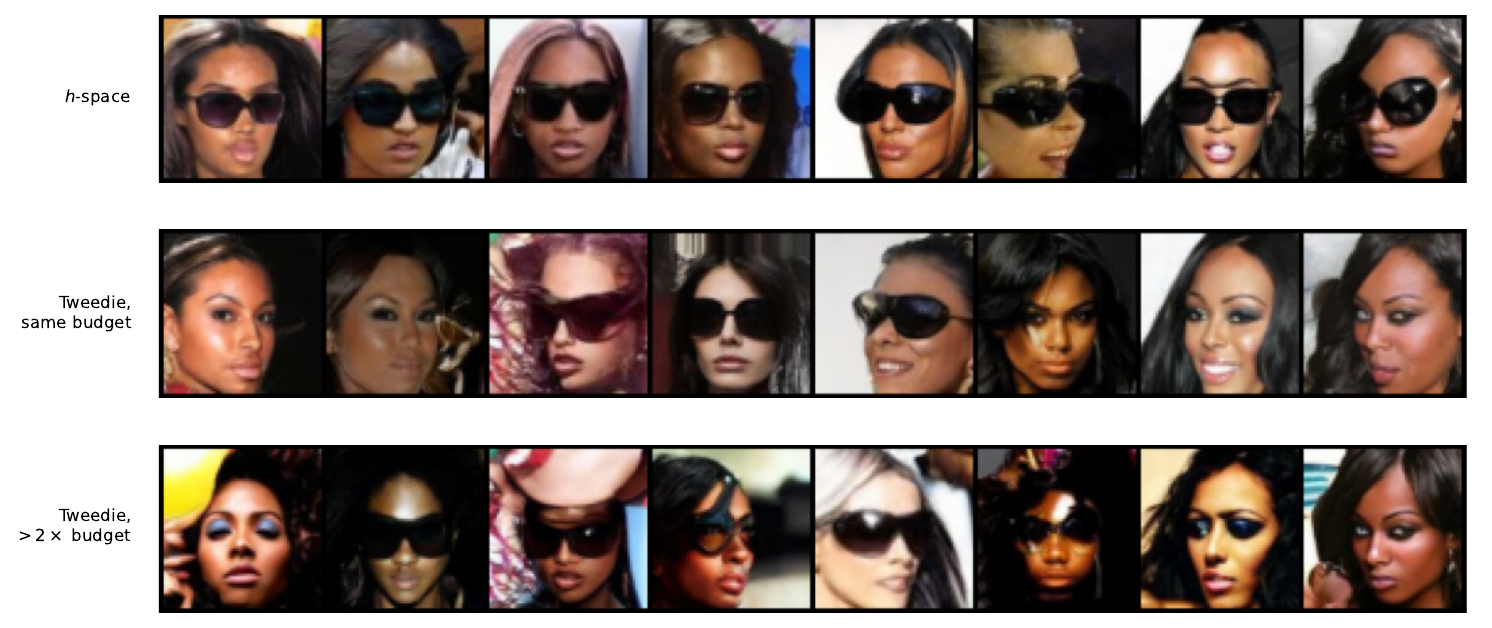}{71}{0.843/{adapter}, 0.500/{pixel,\\same budget}, 0.157/{pixel,\\$2\times$ budget}}
\caption{Same compound target, \emph{dark-skinned female with sunglasses}, three ways of reading the
guidance term. Top: from $\adapt(\hs(x_\sigma);\sigma)$. Middle: from $\imgenc(\xhat)$ at the same sampling
budget. Bottom: from $\imgenc(\xhat)$ with more than twice the budget. Eight samples per row.}
\label{fig:hspace-vs-tweedie}
\end{figure}

Every sample of the top row of \cref{fig:hspace-vs-tweedie} carries both attributes. At an equal
budget the pixel-space score reaches the target for part of the batch only, and more than doubling the
number of steps does not close the gap: several samples still miss the sunglasses, the skin tone, or
both. The failure is not uniform across the batch, which is what an uncontrolled, off-support
evaluation produces: at the noise levels where the guidance acts, $\xhat$ is an average across modes
and its embedding points at no single one of them.

%% ===================== TABLE 1 : h-space vs pixel =====================
\begin{table}[htbp]
  \centering
  \caption{$h$-space against pixel-space fairness term at equal weight $w$; $\FD$ on the steered
  attribute; both share the $w=0$ run, 5000 images, one seed. Bold: the better $\MIND$ of the two on each row.}
  \label{tab:hspace-vs-pixel}
  \vspace{\baselineskip}
  \begin{tabular}{r rrr rrr}
    \toprule
    & \multicolumn{3}{c}{h-space guidance} & \multicolumn{3}{c}{pixel guidance} \\
    \cmidrule(r){2-4}\cmidrule{5-7}
    $w$ & $\MIND$ & $\FD$ & $\FID$ & $\MIND$ & $\FD$ & $\FID$ \\
    \midrule
    \multicolumn{7}{l}{\itshape Eyeglasses attribute ($\FD$ $=0.614$ unguided)} \\
    0 & \shade{15}{19.13} & \shade{15}{0.614} & \shade{29}{16.23} & \shade{100}{19.13} & \shade{15}{0.614} & \shade{100}{16.23} \\
    50 & \shade{29}{\textbf{15.14}} & \shade{29}{0.534} & \shade{58}{13.81} & \shade{86}{19.69} & \shade{29}{0.611} & \shade{86}{16.92} \\
    100 & \shade{58}{\textbf{11.81}} & \shade{43}{0.441} & \shade{72}{12.14} & \shade{72}{19.88} & \shade{43}{0.598} & \shade{72}{17.66} \\
    150 & \shade{86}{\textbf{9.97}} & \shade{58}{0.377} & \shade{100}{11.60} & \shade{58}{20.38} & \shade{58}{0.587} & \shade{58}{18.56} \\
    200 & \shade{100}{\textbf{8.92}} & \shade{72}{0.320} & \shade{86}{11.63} & \shade{43}{20.51} & \shade{72}{0.573} & \shade{43}{19.15} \\
    400 & \shade{72}{\textbf{10.64}} & \shade{86}{0.198} & \shade{43}{15.51} & \shade{29}{20.86} & \shade{86}{0.534} & \shade{29}{20.81} \\
    600 & \shade{43}{\textbf{15.12}} & \shade{100}{0.156} & \shade{15}{21.17} & \shade{15}{20.95} & \shade{100}{0.489} & \shade{15}{22.07} \\
    \midrule
    \multicolumn{7}{l}{\itshape Gender attribute ($\FD$ $=0.0963$ unguided)} \\
    0 & \shade{58}{2.49} & \shade{15}{0.0963} & \shade{72}{6.31} & \shade{58}{2.49} & \shade{15}{0.0963} & \shade{72}{6.31} \\
    2 & \shade{86}{2.31} & \shade{29}{0.0644} & \shade{100}{6.24} & \shade{100}{\textbf{2.21}} & \shade{29}{0.0395} & \shade{100}{6.19} \\
    5 & \shade{100}{2.26} & \shade{43}{0.0302} & \shade{86}{6.25} & \shade{100}{\textbf{2.21}} & \shade{43}{0.0201} & \shade{86}{6.22} \\
    10 & \shade{72}{2.41} & \shade{58}{0.0094} & \shade{58}{6.44} & \shade{72}{\textbf{2.34}} & \shade{58}{0.0097} & \shade{58}{6.39} \\
    20 & \shade{43}{2.85} & \shade{72}{0.0021} & \shade{43}{6.87} & \shade{43}{\textbf{2.64}} & \shade{72}{0.0076} & \shade{43}{6.70} \\
    40 & \shade{29}{3.83} & \shade{100}{0.0010} & \shade{29}{7.69} & \shade{29}{\textbf{3.49}} & \shade{86}{0.0027} & \shade{29}{7.47} \\
    50 & \shade{15}{4.30} & \shade{100}{0.0010} & \shade{15}{8.08} & \shade{15}{\textbf{3.81}} & \shade{100}{0.0018} & \shade{15}{7.79} \\
    \bottomrule
  \end{tabular}
\end{table}

%% ===================== FIGURE 1 : FD vs MIND ==========================
\begin{figure}[htbp]
  \centering
  \begin{tikzpicture}
    \def\fdx#1{((#1)-8)/14*52*1mm}
    \def\fdy#1{((#1)-0.1)/0.56*38*1mm}
    \def\P(#1,#2){({\fdx{#1}},{\fdy{#2}})}
    \foreach \x in {10,12,14,16,18,20}
      \draw[fdgrid] \P(\x,0.1) -- \P(\x,0.66);
    \foreach \y in {0.1,0.2,0.3,0.4,0.5,0.6}
      \draw[fdgrid] \P(8,\y) -- \P(22,\y);
    \draw[fdaxis] \P(8,0.1) -- \P(22,0.1);
    \draw[fdaxis] \P(8,0.1) -- \P(8,0.66);
    \foreach \x in {10,12,14,16,18,20}
      \draw[fdaxis] \P(\x,0.1) -- ++(0,-1.2mm) node[fdtick, below] {\x};
    \draw[fdaxis] \P(8,0.1) -- ++(-1.2mm,0) node[fdtick, left] {0.1};
    \draw[fdaxis] \P(8,0.2) -- ++(-1.2mm,0) node[fdtick, left] {0.2};
    \draw[fdaxis] \P(8,0.3) -- ++(-1.2mm,0) node[fdtick, left] {0.3};
    \draw[fdaxis] \P(8,0.4) -- ++(-1.2mm,0) node[fdtick, left] {0.4};
    \draw[fdaxis] \P(8,0.5) -- ++(-1.2mm,0) node[fdtick, left] {0.5};
    \draw[fdaxis] \P(8,0.6) -- ++(-1.2mm,0) node[fdtick, left] {0.6};
    %% hspace : w = 0, 50, 100, 150, 200, 400, 600
    \draw[thick, DiagTrainL] \P(19.13,0.6143) -- \P(15.14,0.5339) -- \P(11.81,0.4408) -- \P(9.97,0.3768) -- \P(8.92,0.3201) -- \P(10.64,0.1983) -- \P(15.12,0.1564);
    \foreach \px/\py in {19.13/0.6143, 15.14/0.5339, 11.81/0.4408, 9.97/0.3768, 8.92/0.3201, 10.64/0.1983, 15.12/0.1564}
      \node[fdpt, fill=DiagTrainL] at \P(\px,\py) {};
    %% pixel : w = 0, 50, 100, 150, 200, 400, 600
    \draw[thick, DiagEmbL] \P(19.13,0.6143) -- \P(19.69,0.6106) -- \P(19.88,0.5981) -- \P(20.38,0.5869) -- \P(20.51,0.5733) -- \P(20.86,0.5336) -- \P(20.95,0.4894);
    \foreach \px/\py in {19.13/0.6143, 19.69/0.6106, 19.88/0.5981, 20.38/0.5869, 20.51/0.5733, 20.86/0.5336, 20.95/0.4894}
      \node[fdpt, fill=DiagEmbL] at \P(\px,\py) {};
    %% legend
    \foreach \c/\yy/\lab in {DiagTrainL/0.63/h-space, DiagEmbL/0.592/pixel} {
      \draw[thick, \c] \P(13.0,\yy) -- \P(13.3262,\yy);
      \node[fdpt, fill=\c] at \P(13.1631,\yy) {};
      \node[fdleg, right=0.8mm] at \P(13.3262,\yy) {\lab};
    }
    \node[fdbase] at \P(19.13,0.6143) {};
    \node[fdleg, above left=0.6mm] at \P(19.13,0.6143) {unguided};
    \node[fdname, below=4.5mm] at \P(15.0,0.1) {MIND $\rightarrow$ worse};
    \node[fdname, rotate=90, anchor=south] at ({\fdx{8}-7.5mm},{\fdy{0.38}}) {FD $\rightarrow$ worse};
    \node[fdname, above=0.5mm] at \P(15.0,0.66) {Eyeglasses};
  \end{tikzpicture}%
  \hfill
  \begin{tikzpicture}
    \def\fdx#1{((#1)-2.1)/2.4*52*1mm}
    \def\fdy#1{((#1)-0.0)/0.1*38*1mm}
    \def\P(#1,#2){({\fdx{#1}},{\fdy{#2}})}
    \foreach \x in {2.5,3.0,3.5,4.0}
      \draw[fdgrid] \P(\x,0.0) -- \P(\x,0.1);
    \foreach \y in {0,0.02,0.04,0.06,0.08}
      \draw[fdgrid] \P(2.1,\y) -- \P(4.5,\y);
    \draw[fdaxis] \P(2.1,0.0) -- \P(4.5,0.0);
    \draw[fdaxis] \P(2.1,0.0) -- \P(2.1,0.1);
    \foreach \x in {2.5,3.0,3.5,4.0}
      \draw[fdaxis] \P(\x,0.0) -- ++(0,-1.2mm) node[fdtick, below] {\x};
    \draw[fdaxis] \P(2.1,0) -- ++(-1.2mm,0) node[fdtick, left] {0.00};
    \draw[fdaxis] \P(2.1,0.02) -- ++(-1.2mm,0) node[fdtick, left] {0.02};
    \draw[fdaxis] \P(2.1,0.04) -- ++(-1.2mm,0) node[fdtick, left] {0.04};
    \draw[fdaxis] \P(2.1,0.06) -- ++(-1.2mm,0) node[fdtick, left] {0.06};
    \draw[fdaxis] \P(2.1,0.08) -- ++(-1.2mm,0) node[fdtick, left] {0.08};
    %% hspace : w = 0, 2, 5, 10, 20, 40, 50
    \draw[thick, DiagTrainL] \P(2.49,0.0963) -- \P(2.31,0.0644) -- \P(2.26,0.0302) -- \P(2.41,0.0094) -- \P(2.85,0.0021) -- \P(3.83,0.0010) -- \P(4.30,0.0010);
    \foreach \px/\py in {2.49/0.0963, 2.31/0.0644, 2.26/0.0302, 2.41/0.0094, 2.85/0.0021, 3.83/0.0010, 4.30/0.0010}
      \node[fdpt, fill=DiagTrainL] at \P(\px,\py) {};
    %% pixel : w = 0, 2, 5, 10, 20, 40, 50
    \draw[thick, DiagEmbL] \P(2.49,0.0963) -- \P(2.21,0.0395) -- \P(2.21,0.0201) -- \P(2.34,0.0097) -- \P(2.64,0.0076) -- \P(3.49,0.0027) -- \P(3.81,0.0018);
    \foreach \px/\py in {2.49/0.0963, 2.21/0.0395, 2.21/0.0201, 2.34/0.0097, 2.64/0.0076, 3.49/0.0027, 3.81/0.0018}
      \node[fdpt, fill=DiagEmbL] at \P(\px,\py) {};
    %% legend
    \foreach \c/\yy/\lab in {DiagTrainL/0.096/h-space, DiagEmbL/0.089/pixel} {
      \draw[thick, \c] \P(3.5,\yy) -- \P(3.5559,\yy);
      \node[fdpt, fill=\c] at \P(3.5280,\yy) {};
      \node[fdleg, right=0.8mm] at \P(3.5559,\yy) {\lab};
    }
    \node[fdbase] at \P(2.49,0.0963) {};
    \node[fdleg, right=1mm] at \P(2.49,0.0963) {unguided};
    \node[fdname, below=4.5mm] at \P(3.3,0.0) {MIND $\rightarrow$ worse};
    \node[fdname, rotate=90, anchor=south] at ({\fdx{2.1}-9mm},{\fdy{0.05}}) {FD $\rightarrow$ worse};
    \node[fdname, above=0.5mm] at \P(3.3,0.1) {Gender};
  \end{tikzpicture}
  \caption{FD--$\MIND$ trade-off, h-space versus pixel-space fairness term; each
  polyline follows its own sweep in order of increasing guidance weight $w$, and
  the white square is the shared unguided model. Left: on eyeglasses the two
  curves do not overlap, pixel guidance stays pinned in the top-right corner,
  drifting \emph{rightwards} as $w$ grows, while h-space sweeps down and to the
  left, then turns back up past $w=200$ as fidelity starts to cost, so no pixel
  setting matches any h-space setting on both axes. Right: on gender the two
  sweeps trace the same short hook, pixel running marginally inside h-space;
  both bottom out in $\MIND$ around $w=5$ and then climb steeply as $\FD$ is driven
  to zero.}
  \label{fig:hspace-vs-pixel}
\end{figure}
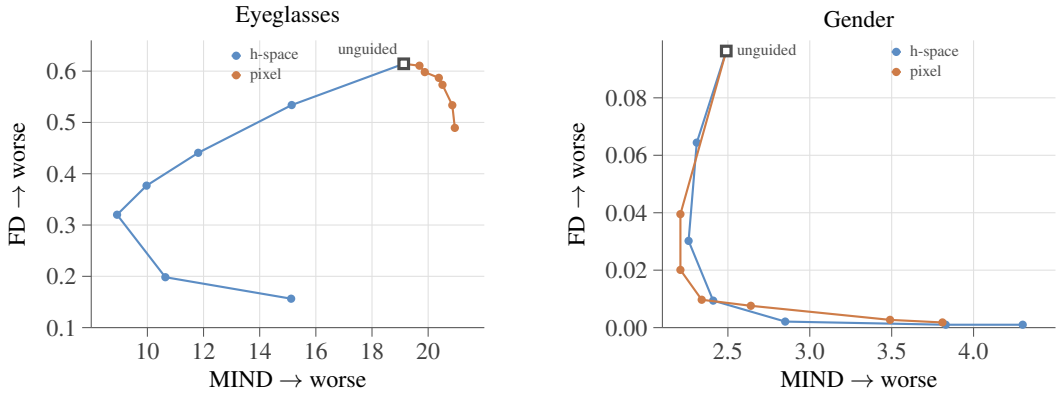

On the strongly skewed eyeglasses attribute the two estimates are not comparable in effect. The
$h$-space score drives the discrepancy from $0.614$ to $0.156$ and improves $\MIND$ and $\FID$
until $w = 200$ ($\MIND$ $8.9$, $\FID$ $11.6$); the pixel-space score stalls at a discrepancy of
$0.489$, never entering the range the $h$-space covers, and pays for that small gain with a rising
$\MIND$ ($19.1 \to 21.0$) and a rising $\FID$ ($16.2 \to 22.1$). On the nearly balanced gender
attribute the ranking reverses but the margin collapses: the pixel-space score takes $\MIND$ on every
guided row, by $2\%$ at the selected weight $w = 5$ and $13\%$ at $w = 50$, against a seed-to-seed standard deviation of $5$ to $10\%$ of $\MIND$ measured on the same
configuration, so the two are indistinguishable there. The $h$-space score is also cheaper, its
gradient stopping at the mid-block instead of traversing the full denoiser and the encoder.

\subsection{Fairness weight}
\label{app:sweeps}

\cref{tab:hspace-vs-pixel} doubles as the fairness-weight sweep. On both attributes the
discrepancy falls monotonically with $w_{\text{fair}}$, but quality does not: $\MIND$ and $\FID$
first improve, as the guidance corrects the imbalance, reach a minimum, and then degrade. On
eyeglasses the minimum is at $w_{\text{fair}} = 200$; between it and $w_{\text{fair}} = 600$, $\MIND$
rises from $8.9$ to $15.1$ and $\FID$ from $11.6$ to $21.2$ while the discrepancy gains only its last
$0.16$. On gender the minimum is at $w_{\text{fair}} \approx 5$ and the climb past it is steep, $\MIND$
reaching $4.3$ at $w_{\text{fair}} = 50$ for a discrepancy already at $0.001$. The setting that
maximises quality is therefore not the setting that maximises parity, and the useful operating range
ends well before parity is reached. Steering eyeglasses also moves the gender balance: the gender
discrepancy rises from $0.096$ to $0.343$ at $w_{\text{fair}} = 600$, a collateral shift that a
multi-attribute target would have to control.

\subsection{Diversity weight}
\label{app:divsweeps}

Four diversity terms are compared from here on: \code{PerturbProj}, the term of
\cref{eq:perturbproj}; \code{SigLIPMS}, its variant of \cref{app:siglipms}; and the two baselines
of \cref{app:details-baselines}, \code{SGMS} and \code{avgKNN}. With the fairness weight fixed at
the value selected above, each term is swept over its own weight; $w_{\mathrm{div}} = 0$ reproduces the fairness-only run in every block, which serves as a
consistency check, and $\star$ marks the setting carried forward to \cref{app:fairdiverse}. 

\cref{tab:div-eyeglasses,tab:div-gender} use 2000 images per setting, so their FID and MIND are inflated relative to \cref{tab:hspace-vs-pixel,tab:fairdiverse-5seeds} and comparable only within a sweep; the selected settings are re-evaluated at 5,000 images and five seeds in \cref{tab:fairdiverse-5seeds}.

%% ============================
\begin{table}[htbp]
  \centering
  \caption{Diversity-weight sweep on eyeglasses, fairness weight fixed at $200$. $w_{\mathrm{div}}=0$ is
  the fairness-only baseline, 2000 images, one seed; $\star$ marks the setting carried to \cref{tab:fairdiverse-5seeds}.}
  \label{tab:div-eyeglasses}
  \vspace{\baselineskip}
  \small
  \setlength{\tabcolsep}{4pt}
  \begin{tabular}{l r rrrrrrrr}
    \toprule
    Method & $w_{\mathrm{div}}$ & \lo{$\MIND$} & \lo{$\FD$} & \lo{$\FID$} & \hi{Vendi} & \hi{Prec.} & \hi{Rec.} & \hi{Cov.} & \hi{Dens.} \\
    \midrule
    \code{avgKNN} & 0 & \shade{49}{10.76} & \shade{15}{0.336} & \best{19.21} & \shade{15}{6.47} & \best{0.814} & \shade{15}{0.679} & \best{0.893} & \best{1.218} \\
      & 2 & \shade{66}{10.56} & \shade{32}{0.332} & \best{19.21} & \shade{32}{6.52} & \shade{83}{0.812} & \shade{32}{0.689} & \best{0.893} & \shade{83}{1.208} \\
      & 5$^{\star}$ & \best{10.44} & \shade{66}{0.330} & \shade{66}{19.39} & \shade{49}{6.58} & \shade{66}{0.805} & \shade{49}{0.690} & \best{0.893} & \shade{66}{1.190} \\
      & 10 & \shade{83}{10.46} & \shade{66}{0.330} & \shade{49}{19.78} & \shade{66}{6.69} & \shade{49}{0.795} & \best{0.697} & \shade{49}{0.880} & \shade{49}{1.159} \\
      & 20 & \shade{32}{11.24} & \shade{83}{0.317} & \shade{32}{20.77} & \shade{83}{6.88} & \shade{32}{0.778} & \shade{66}{0.692} & \shade{32}{0.870} & \shade{32}{1.131} \\
      & 40 & \shade{15}{14.72} & \best{0.309} & \shade{15}{23.80} & \best{7.31} & \shade{15}{0.765} & \shade{83}{0.696} & \shade{15}{0.835} & \shade{15}{1.035} \\
    \midrule
    \code{SGMS} & 0 & \shade{66}{10.75} & \shade{15}{0.336} & \best{19.18} & \shade{15}{6.48} & \best{0.815} & \shade{15}{0.683} & \best{0.892} & \best{1.217} \\
      & 10 & \best{10.64} & \shade{32}{0.332} & \shade{83}{19.52} & \shade{32}{6.57} & \shade{83}{0.807} & \shade{49}{0.695} & \shade{83}{0.887} & \shade{83}{1.186} \\
      & 20$^{\star}$ & \shade{83}{10.68} & \shade{66}{0.327} & \shade{66}{19.93} & \shade{49}{6.63} & \shade{66}{0.797} & \shade{32}{0.694} & \shade{66}{0.881} & \shade{66}{1.171} \\
      & 40 & \shade{49}{10.84} & \shade{66}{0.327} & \shade{49}{20.72} & \shade{66}{6.78} & \shade{49}{0.778} & \shade{66}{0.696} & \shade{49}{0.867} & \shade{49}{1.120} \\
      & 60 & \shade{32}{11.29} & \shade{83}{0.323} & \shade{32}{21.75} & \shade{83}{6.97} & \shade{32}{0.766} & \shade{83}{0.704} & \shade{32}{0.860} & \shade{32}{1.085} \\
      & 80 & \shade{15}{12.08} & \best{0.316} & \shade{15}{22.91} & \best{7.12} & \shade{15}{0.756} & \best{0.714} & \shade{15}{0.852} & \shade{15}{1.047} \\
    \midrule
    \code{SigLIPMS} & 0 & \shade{47}{10.77} & \shade{15}{0.336} & \shade{79}{19.21} & \shade{15}{6.47} & \best{0.813} & \shade{89}{0.682} & \shade{89}{0.894} & \best{1.215} \\
      & 20 & \shade{68}{9.89} & \shade{26}{0.327} & \shade{89}{19.02} & \shade{26}{6.54} & \shade{89}{0.810} & \best{0.685} & \best{0.897} & \shade{89}{1.210} \\
      & 40 & \shade{89}{9.43} & \shade{36}{0.317} & \best{18.92} & \shade{36}{6.59} & \shade{68}{0.806} & \shade{58}{0.672} & \shade{79}{0.889} & \shade{79}{1.200} \\
      & 80 & \best{9.07} & \shade{47}{0.298} & \shade{79}{19.21} & \shade{47}{6.72} & \shade{68}{0.806} & \shade{68}{0.675} & \shade{68}{0.886} & \shade{68}{1.180} \\
      & 120$^{\star}$ & \shade{79}{9.49} & \shade{58}{0.287} & \shade{58}{19.99} & \shade{58}{6.85} & \shade{79}{0.807} & \shade{79}{0.678} & \shade{58}{0.879} & \shade{47}{1.165} \\
      & 160 & \shade{58}{10.26} & \shade{68}{0.283} & \shade{47}{20.72} & \shade{68}{6.94} & \shade{47}{0.800} & \shade{26}{0.657} & \shade{47}{0.875} & \shade{36}{1.153} \\
      & 200 & \shade{36}{11.21} & \shade{79}{0.280} & \shade{36}{21.56} & \shade{79}{7.07} & \shade{36}{0.795} & \shade{47}{0.669} & \shade{36}{0.863} & \shade{58}{1.168} \\
      & 240 & \shade{26}{12.58} & \shade{89}{0.277} & \shade{26}{22.57} & \shade{89}{7.15} & \shade{15}{0.790} & \shade{36}{0.668} & \shade{26}{0.855} & \shade{26}{1.135} \\
      & 300 & \shade{15}{15.05} & \best{0.267} & \shade{15}{24.31} & \best{7.26} & \shade{26}{0.794} & \shade{15}{0.655} & \shade{15}{0.830} & \shade{15}{1.094} \\
    \midrule
    \code{PerturbProj} & 0 & \shade{15}{10.77} & \shade{15}{0.336} & \shade{91}{19.21} & \shade{15}{6.47} & \best{0.813} & \shade{81}{0.682} & \shade{91}{0.894} & \best{1.215} \\
      & 20 & \shade{24}{10.60} & \shade{24}{0.334} & \best{19.19} & \shade{24}{6.50} & \shade{91}{0.812} & \shade{62}{0.680} & \shade{62}{0.889} & \shade{91}{1.212} \\
      & 40 & \shade{34}{10.48} & \shade{34}{0.332} & \shade{91}{19.21} & \shade{34}{6.52} & \shade{81}{0.811} & \shade{81}{0.682} & \shade{43}{0.886} & \shade{81}{1.209} \\
      & 80 & \shade{43}{10.31} & \shade{43}{0.326} & \shade{72}{19.29} & \shade{43}{6.55} & \shade{53}{0.804} & \best{0.688} & \shade{43}{0.886} & \shade{72}{1.201} \\
      & 120 & \shade{53}{10.27} & \shade{53}{0.323} & \shade{62}{19.42} & \shade{53}{6.57} & \shade{43}{0.802} & \shade{91}{0.684} & \shade{62}{0.889} & \shade{53}{1.198} \\
      & 160 & \shade{62}{10.22} & \shade{62}{0.320} & \shade{53}{19.49} & \shade{62}{6.59} & \shade{72}{0.806} & \shade{53}{0.678} & \shade{72}{0.890} & \shade{43}{1.196} \\
      & 200 & \shade{72}{10.18} & \shade{72}{0.317} & \shade{43}{19.58} & \shade{72}{6.61} & \shade{62}{0.805} & \shade{43}{0.677} & \shade{81}{0.892} & \shade{62}{1.200} \\
      & 240 & \shade{81}{10.09} & \shade{81}{0.313} & \shade{34}{19.62} & \shade{81}{6.65} & \shade{43}{0.802} & \shade{15}{0.669} & \best{0.895} & \shade{34}{1.187} \\
      & 300$^{\star}$ & \best{10.05} & \shade{91}{0.309} & \shade{24}{19.87} & \shade{91}{6.72} & \shade{15}{0.787} & \shade{34}{0.673} & \shade{15}{0.881} & \shade{24}{1.185} \\
      & 350 & \shade{91}{10.08} & \best{0.307} & \shade{15}{20.09} & \best{6.77} & \shade{24}{0.790} & \shade{24}{0.672} & \shade{24}{0.883} & \shade{15}{1.177} \\
    \bottomrule
  \end{tabular}
\end{table}

%% ================================
\begin{table}[htbp]
  \centering
  \caption{Diversity-weight sweep, on gender, fairness weight fixed at $10$. $w_{\mathrm{div}}=0$ is
  the fairness-only baseline, 2000 images, one seed; $\star$ marks the setting carried to \cref{tab:fairdiverse-5seeds}.}
  \label{tab:div-gender}
  \vspace{\baselineskip}
  \small
  \setlength{\tabcolsep}{4pt}
  \begin{tabular}{l r rrrrrrrr}
    \toprule
    Method & $w_{\mathrm{div}}$ & \lo{$\MIND$} & \lo{$\FD$} & \lo{$\FID$} & \hi{Vendi} & \hi{Prec.} & \hi{Rec.} & \hi{Cov.} & \hi{Dens.} \\
    \midrule
    \code{avgKNN} & 0 & \shade{49}{2.76} & \shade{32}{0.0104} & \shade{66}{12.41} & \shade{15}{6.18} & \shade{83}{0.829} & \shade{32}{0.723} & \shade{83}{0.970} & \best{1.133} \\
      & 2 & \shade{66}{2.65} & \shade{49}{0.0075} & \shade{83}{12.40} & \shade{32}{6.24} & \shade{66}{0.828} & \shade{32}{0.723} & \shade{66}{0.969} & \shade{83}{1.124} \\
      & 5$^{\star}$ & \best{2.54} & \shade{66}{0.0039} & \best{12.39} & \shade{49}{6.32} & \best{0.832} & \shade{66}{0.732} & \best{0.973} & \shade{66}{1.113} \\
      & 10 & \shade{83}{2.56} & \best{0.0017} & \shade{49}{12.54} & \shade{66}{6.45} & \shade{32}{0.820} & \shade{83}{0.734} & \shade{32}{0.963} & \shade{49}{1.094} \\
      & 20 & \shade{32}{3.42} & \shade{83}{0.0019} & \shade{32}{13.28} & \shade{83}{6.72} & \shade{49}{0.824} & \shade{66}{0.732} & \shade{49}{0.967} & \shade{32}{1.060} \\
      & 40 & \shade{15}{8.34} & \shade{15}{0.0121} & \shade{15}{16.81} & \best{7.26} & \shade{15}{0.792} & \best{0.749} & \shade{15}{0.902} & \shade{15}{0.936} \\
    \midrule
    \code{SGMS} & 0 & \shade{43}{2.77} & \shade{15}{0.0104} & \best{12.41} & \shade{15}{6.18} & \best{0.830} & \shade{43}{0.724} & \shade{58}{0.969} & \best{1.135} \\
      & 5 & \shade{72}{2.66} & \shade{29}{0.0075} & \shade{86}{12.42} & \shade{29}{6.22} & \shade{86}{0.828} & \shade{15}{0.718} & \shade{86}{0.971} & \shade{86}{1.117} \\
      & 10 & \shade{86}{2.60} & \shade{43}{0.0068} & \shade{72}{12.46} & \shade{43}{6.27} & \shade{72}{0.826} & \shade{29}{0.722} & \best{0.973} & \shade{72}{1.106} \\
      & 20$^{\star}$ & \best{2.51} & \best{0.0010} & \shade{58}{12.54} & \shade{58}{6.38} & \shade{58}{0.824} & \shade{58}{0.738} & \shade{72}{0.970} & \shade{58}{1.076} \\
      & 40 & \shade{58}{2.72} & \shade{58}{0.0039} & \shade{43}{12.88} & \shade{72}{6.52} & \shade{43}{0.808} & \shade{86}{0.753} & \shade{43}{0.964} & \shade{43}{1.035} \\
      & 60 & \shade{29}{3.29} & \shade{72}{0.0017} & \shade{29}{13.47} & \shade{86}{6.68} & \shade{29}{0.786} & \shade{72}{0.750} & \shade{29}{0.960} & \shade{29}{0.985} \\
      & 80 & \shade{15}{4.47} & \shade{86}{0.0012} & \shade{15}{14.58} & \best{6.90} & \shade{15}{0.771} & \best{0.756} & \shade{29}{0.960} & \shade{15}{0.949} \\
    \midrule
    \code{SigLIPMS} & 0 & \shade{66}{2.76} & \shade{66}{0.0104} & \best{12.41} & \shade{15}{6.18} & \shade{83}{0.829} & \shade{66}{0.723} & \shade{66}{0.970} & \shade{83}{1.133} \\
      & 20$^{\star}$ & \best{2.54} & \best{0.0017} & \shade{83}{12.42} & \shade{32}{6.28} & \best{0.836} & \shade{83}{0.724} & \best{0.972} & \best{1.136} \\
      & 40 & \shade{83}{2.58} & \shade{83}{0.0031} & \shade{66}{12.59} & \shade{49}{6.39} & \shade{66}{0.827} & \shade{66}{0.723} & \shade{83}{0.971} & \shade{66}{1.111} \\
      & 80 & \shade{49}{3.71} & \shade{49}{0.0128} & \shade{49}{13.60} & \shade{66}{6.65} & \shade{49}{0.821} & \best{0.726} & \shade{49}{0.966} & \shade{49}{1.085} \\
      & 120 & \shade{32}{5.81} & \shade{32}{0.0186} & \shade{32}{14.99} & \shade{83}{6.83} & \shade{32}{0.819} & \shade{15}{0.707} & \shade{32}{0.945} & \shade{32}{1.071} \\
      & 160 & \shade{15}{9.39} & \shade{15}{0.0360} & \shade{15}{17.11} & \best{7.01} & \shade{15}{0.818} & \shade{32}{0.714} & \shade{15}{0.939} & \shade{15}{1.038} \\
    \midrule
    \code{PerturbProj} & 0 & \shade{15}{2.77} & \shade{15}{0.0104} & \shade{51}{12.41} & \shade{15}{6.18} & \shade{39}{0.830} & \shade{76}{0.726} & \shade{15}{0.970} & \best{1.134} \\
      & 10 & \shade{27}{2.74} & \shade{27}{0.0082} & \shade{64}{12.40} & \shade{27}{6.19} & \shade{39}{0.830} & \shade{64}{0.725} & \shade{39}{0.971} & \best{1.134} \\
      & 20 & \shade{39}{2.70} & \shade{39}{0.0068} & \shade{88}{12.39} & \shade{39}{6.21} & \best{0.834} & \shade{64}{0.725} & \shade{51}{0.972} & \shade{76}{1.133} \\
      & 40 & \shade{51}{2.66} & \shade{51}{0.0053} & \best{12.38} & \shade{51}{6.23} & \shade{51}{0.832} & \shade{27}{0.723} & \shade{76}{0.973} & \shade{64}{1.130} \\
      & 80 & \shade{76}{2.58} & \shade{76}{0.0039} & \shade{88}{12.39} & \shade{64}{6.26} & \best{0.834} & \shade{88}{0.728} & \best{0.975} & \shade{51}{1.124} \\
      & 120$^{\star}$ & \shade{88}{2.56} & \shade{88}{0.0031} & \shade{39}{12.47} & \shade{76}{6.31} & \shade{76}{0.833} & \shade{15}{0.719} & \shade{88}{0.974} & \shade{39}{1.107} \\
      & 160 & \best{2.55} & \shade{76}{0.0039} & \shade{27}{12.54} & \shade{88}{6.36} & \shade{76}{0.833} & \best{0.732} & \shade{76}{0.973} & \shade{39}{1.107} \\
      & 200 & \shade{64}{2.60} & \best{0.0017} & \shade{15}{12.63} & \best{6.42} & \shade{15}{0.828} & \shade{39}{0.724} & \shade{39}{0.971} & \shade{15}{1.095} \\
    \bottomrule
  \end{tabular}
\end{table}

The columns disagree, so no setting dominates. On eyeglasses (\cref{tab:div-eyeglasses}) $\FID$
and density are best at or near $w_{\mathrm{div}} = 0$ while discrepancy and Vendi keep improving to
the end of every sweep; the retained value is a compromise, and for \code{SGMS} and \code{SigLIPMS}
it is deliberately past the $\MIND$ minimum, buying discrepancy and Vendi at a small $\MIND$ cost.
Only \code{SigLIPMS} improves $\MIND$, discrepancy and $\FID$ together over its own baseline, at
$w_{\mathrm{div}}$ between $40$ and $80$; the other three buy discrepancy and Vendi while giving up
$\FID$, precision, coverage and density, the four moving as one. On gender
(\cref{tab:div-gender}) all four terms improve $\MIND$ and discrepancy over their baseline, and
coverage holds or rises while density falls only slightly, which is why the final table shows no
fidelity penalty there. Every sweep is unimodal in $\MIND$ and turns sharply past its optimum
(\code{avgKNN} reaches $8.34$ at $w_{\mathrm{div}} = 40$, \code{SigLIPMS} $9.39$ at $160$), so the
useful range is narrow.

\subsection{Diversity on top of the fairness term at the selected setting}
\label{app:fairdiverse}

\cref{tab:fairdiverse-5seeds} compares the four diversity terms at the fairness weight of
\cref{app:sweeps} and the diversity weight of \cref{app:divsweeps}, as means and standard deviations
over five sampling seeds of $5,000$ images; it is the table behind the diversity findings of
\cref{sec:res-ablations}.
\cref{tab:fairdiverse-50k} repeats the comparison with $50,000$ images per configuration and a
single seed.

\begin{table}[htbp]
  \centering
  \caption{Diversity terms on top of the fairness term at the selected weights, 5000 images, 5 seeds; time is wall-clock
  generation of one run and is not ranked.}
  \label{tab:fairdiverse-5seeds}
  \vspace{\baselineskip}
  \scriptsize
  \setlength{\tabcolsep}{5pt}
  \begin{tabular}{l rrrrr}
    \toprule
    & time (s) & \lo{$\FD$} & \lo{$\FID$} & \lo{$\MIND$} & \hi{Vendi} \\
    \midrule
    \multicolumn{6}{l}{\itshape Gender, $w_{\text{fair}} = 5$} \\
    \method{} & $1613 \pm 665$ & \shade{15}{$0.031 \pm 0.004$} & \shade{15}{$6.59 \pm 0.04$} & \shade{15}{$2.52 \pm 0.13$} & \shade{15}{$6.26 \pm 0.03$} \\
    \quad $+$ \code{SGMS} & $1380 \pm 5$ & \shade{36}{$0.023 \pm 0.005$} & \shade{58}{$6.46 \pm 0.05$} & \shade{79}{$2.05 \pm 0.14$} & \best{$6.46 \pm 0.02$} \\
    \quad $+$ \code{avgKNN} & $1635 \pm 713$ & \shade{58}{$0.021 \pm 0.003$} & \best{$6.40 \pm 0.11$} & \shade{58}{$2.06 \pm 0.15$} & \shade{58}{$6.42 \pm 0.03$} \\
    \quad $+$ \code{SigLIPMS} & $2179 \pm 962$ & \shade{79}{$0.011 \pm 0.004$} & \shade{36}{$6.48 \pm 0.07$} & \shade{36}{$2.13 \pm 0.17$} & \shade{36}{$6.41 \pm 0.04$} \\
    \quad $+$ \code{PerturbProj} & $1476 \pm 592$ & \best{$0.006 \pm 0.003$} & \shade{79}{$6.44 \pm 0.10$} & \best{$2.03 \pm 0.15$} & \shade{79}{$6.43 \pm 0.04$} \\
    \midrule
    \multicolumn{6}{l}{\itshape Eyeglasses, $w_{\text{fair}} = 200$} \\
    \method{} & $1619 \pm 662$ & \shade{15}{$0.304 \pm 0.004$} & \shade{79}{$11.47 \pm 0.12$} & \shade{36}{$9.06 \pm 0.48$} & \shade{15}{$6.62 \pm 0.02$} \\
    \quad $+$ \code{SGMS} & $1418 \pm 51$ & \shade{58}{$0.300 \pm 0.002$} & \shade{58}{$11.87 \pm 0.10$} & \shade{79}{$8.81 \pm 0.44$} & \shade{58}{$6.80 \pm 0.01$} \\
    \quad $+$ \code{avgKNN} & $1599 \pm 678$ & \shade{58}{$0.300 \pm 0.003$} & \best{$11.46 \pm 0.09$} & \best{$8.64 \pm 0.41$} & \shade{36}{$6.73 \pm 0.03$} \\
    \quad $+$ \code{SigLIPMS} & $2162 \pm 936$ & \best{$0.270 \pm 0.005$} & \shade{15}{$12.65 \pm 0.29$} & \shade{15}{$9.14 \pm 0.83$} & \best{$7.03 \pm 0.04$} \\
    \quad $+$ \code{PerturbProj} & $1703 \pm 755$ & \shade{79}{$0.288 \pm 0.004$} & \shade{36}{$12.38 \pm 0.16$} & \shade{58}{$8.82 \pm 0.56$} & \shade{79}{$6.91 \pm 0.03$} \\
    \bottomrule
  \end{tabular}
\end{table}

\begin{table}[htbp]
  \centering
  \caption{As \cref{tab:fairdiverse-5seeds}, with $50\,000$ images per configuration and one seed.}
  \label{tab:fairdiverse-50k}
  \vspace{\baselineskip}
  \small
  \setlength{\tabcolsep}{5pt}
  \begin{tabular}{l rrrrr}
    \toprule
    & time (s) & \lo{$\FD$} & \lo{$\FID$} & \lo{$\MIND$} & \hi{Vendi} \\
    \midrule
    \multicolumn{6}{l}{\itshape Gender, $w_{\text{fair}} = 5$} \\
    \method{} & 11065 & \shade{15}{0.024} & \shade{15}{2.48} & \shade{15}{2.14} & \shade{15}{6.33} \\
    \quad $+$ \code{SGMS} & 15747 & \shade{36}{0.018} & \shade{58}{2.24} & \shade{79}{1.61} & \best{6.55} \\
    \quad $+$ \code{avgKNN} & 10942 & \shade{58}{0.015} & \best{2.22} & \shade{58}{1.63} & \shade{58}{6.50} \\
    \quad $+$ \code{SigLIPMS} & 15560 & \shade{79}{0.008} & \shade{36}{2.29} & \shade{36}{1.72} & \shade{36}{6.48} \\
    \quad $+$ \code{PerturbProj} & 11417 & \best{0.002} & \shade{79}{2.23} & \best{1.59} & \shade{79}{6.52} \\
    \midrule
    \multicolumn{6}{l}{\itshape Eyeglasses, $w_{\text{fair}} = 200$} \\
    \method{} & 11120 & \shade{15}{0.311} & \shade{79}{7.50} & \shade{36}{8.78} & \shade{15}{6.69} \\
    \quad $+$ \code{SGMS} & 13829 & \shade{58}{0.307} & \shade{58}{7.84} & \shade{58}{8.42} & \shade{58}{6.88} \\
    \quad $+$ \code{avgKNN} & 10923 & \shade{58}{0.307} & \best{7.47} & \shade{79}{8.32} & \shade{36}{6.81} \\
    \quad $+$ \code{SigLIPMS} & 14712 & \best{0.272} & \shade{15}{8.45} & \shade{15}{8.79} & \best{7.08} \\
    \quad $+$ \code{PerturbProj} & 12431 & \shade{79}{0.293} & \shade{36}{8.13} & \best{8.28} & \shade{79}{7.00} \\
    \bottomrule
  \end{tabular}
\end{table}

On gender the gains are unambiguous. All four terms lower $\MIND$ from $2.52$ to $2.03$--$2.13$,
two to three standard deviations, and improve the discrepancy, $\FID$ and Vendi at the same time.
\code{PerturbProj} reaches the best discrepancy, $0.006$ against $0.031$, and the best $\MIND$; the
four terms are within one standard deviation of each other on $\MIND$, so the choice between them is
driven by the discrepancy. On eyeglasses the terms buy discrepancy ($0.304 \to 0.270$ for
\code{SigLIPMS}, $0.288$ for \code{PerturbProj}, both many standard deviations) and Vendi ($6.62 \to
7.03$), while the five $\MIND$ means, $8.64$ to $9.14$, all lie within one standard deviation
($0.4$ to $0.8$) of one another: on this attribute the diversity terms cannot be ranked on $\MIND$, and
the two semantic terms do not improve it. $\FID$ rises from $11.47$ to $12.4$--$12.7$ for the two
semantic terms and stays put for \code{avgKNN}. Since the reference is balanced on the attribute,
a lower discrepancy alone would lower $\FID$, so the rise is a genuine departure from the reference
manifold: the semantic terms leave it, which is the price of the coverage they recover, whereas
the density term, measured against real embeddings, stays on it. Vendi, which needs no reference,
moves the other way, from $6.62$ to $7.03$, and that recovered spread is the quantity the fairness
term had taken away. The $50\,000$-sample run reproduces every one of these
orderings; $\FID$ falls by a factor of two to three with the larger sample, as expected from its
bias at small sample size \citep{chong2020unbiasedfid}, $\MIND$ falls by $0.3$ to $0.6$ and Vendi
barely moves. Generation time is dominated by the sampler; in the 50,000-image runs, the diversity terms add up to approximately $43\%$ overhead, with the largest increases for SGMS and \code{SigLIPMS}.

\subsection{Diversity terms alone}
\label{app:alone}

The earliest runs measure each diversity term without the fairness term, against the $10\,000$ CelebA
training images with the highest average distance to their five nearest neighbours in ResNet-50
feature space, following \citet{sehwag2022lowdensity}: against this set a term scores well only if it
reaches the rare part of the data. These runs predate the protocol of the rest of this appendix:
$\sigma_{\max} = 20$, $1000$ images per setting, prompt \emph{``a photo of a man''}, and every
metric below, $\MIND$ included, is against this rare subset and not against the balanced reference
of the other sections.

Selecting each term at its lowest $\MIND$ would be a mistake. \code{SGMS} reaches $\MIND = 27.6$
at $\sigma_s = 0.4$, $w_{\mathrm{div}} = 200$ (\cref{tab:sgms-pr}), but its precision has fallen
from $0.97$ to $0.77$ by then and the images at that setting carry visible artefacts: the term has
left the model's manifold, and the metric, which only asks how close the samples are to the rare
reference images, does not tell that apart from reaching the tail of the data.
\cref{tab:diversity-alone-mind} therefore keeps, for each term, the setting with the lowest $\MIND$
among those whose precision stays at or above $0.95$, a floor $0.02$ below the unguided model.
Under that floor \code{avgKNN} reaches the lowest $\MIND$, $62.7$, with the best coverage and
Vendi; \code{SigLIPMS} comes second at $71.2$ with the best density, \code{PerturbProj} third, and
\code{SGMS}, the term with the largest recall gain, last at $101.0$, because its recall is bought with
precision, which the floor forbids.

\begin{table}[htbp]
  \centering
  \caption{Each diversity term alone, at its lowest-$\MIND$ setting with precision $\geq 0.95$ in
  \cref{tab:avgknn-pr,tab:pproj-pr,tab:sgms-pr,tab:siglipms-pr}.}
  \label{tab:diversity-alone-mind}
  \vspace{\baselineskip}
  \small
  \setlength{\tabcolsep}{4.5pt}
  \begin{tabular}{l l rrrrrr}
    \toprule
    Method & Setting & \lo{$\MIND$} & \hi{Prec.} & \hi{Rec.} & \hi{Cov.} & \hi{Dens.} & \hi{Vendi} \\
    \midrule
    Original & --- & 124.10 & 0.972 & 0.389 & 0.809 & 3.010 & 5.94 \\
    \midrule
    \code{avgKNN} & $\kappa=50,\ w_{\mathrm{div}}=30$ & \best{62.71} & \shade{43}{0.954} & \shade{72}{0.425} & \best{0.874} & \shade{43}{2.875} & \best{7.33} \\
    \code{SGMS} & $\sigma_s=0.8,\ w_{\mathrm{div}}=40$ & \shade{15}{100.96} & \shade{72}{0.955} & \best{0.468} & \shade{43}{0.841} & \shade{15}{2.783} & \shade{43}{6.63} \\
    \code{PerturbProj} & $\sigma_s=0.4,\ w_{\mathrm{div}}=250$ & \shade{43}{92.34} & \best{0.958} & \shade{43}{0.386} & \shade{15}{0.837} & \shade{72}{2.909} & \shade{15}{6.56} \\
    \code{SigLIPMS} & $w_{\mathrm{div}}=250$ & \shade{72}{71.17} & \shade{15}{0.953} & \shade{15}{0.332} & \shade{72}{0.865} & \best{3.076} & \shade{72}{6.96} \\
    \bottomrule
  \end{tabular}
\end{table}

Two things qualify this ranking. The measurement favours \code{avgKNN}: its score is a distance to
a bank of SigLIP embeddings of real CelebA images (\cref{app:densitybaselines}), and the reference
set is the rarest part of the same data, so the term is pushed towards the region it is then
evaluated on, while the two semantic terms measure the model's own posterior spread and are given no
such hint. And the ranking does not survive the fairness term. Once $w_{\text{fair}} > 0$
(\cref{tab:fairdiverse-5seeds}) \code{PerturbProj} and \code{SigLIPMS} reach the lowest
discrepancy on both attributes, \code{PerturbProj} the lowest $\MIND$ on gender and
\code{SigLIPMS} the highest Vendi on eyeglasses, and \code{avgKNN} keeps only $\FID$ on
eyeglasses. The single-term table selects a diversity score; the combined tables select a debiasing
method, and the two disagree.

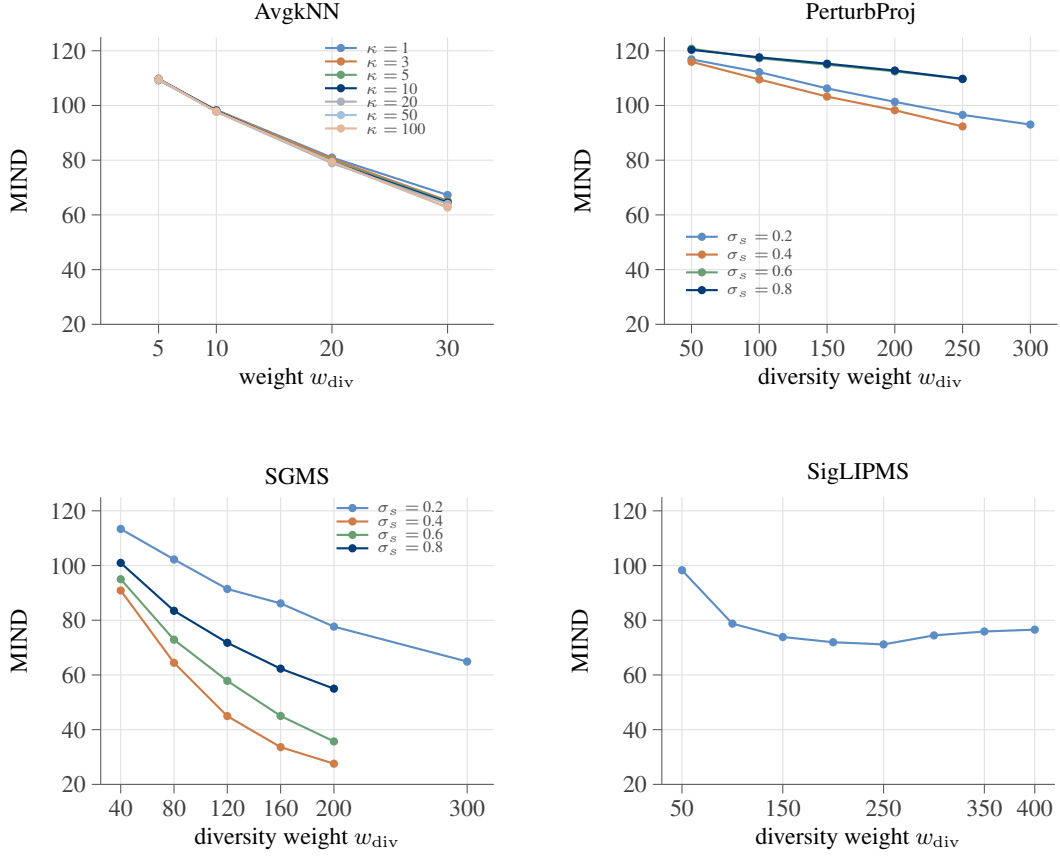
\begin{figure}[htbp]
  \centering
  \begin{tikzpicture}
    \def\fdx#1{((#1)-0)/34*52*1mm}
    \def\fdy#1{((#1)-20)/105*38*1mm}
    \def\P(#1,#2){({\fdx{#1}},{\fdy{#2}})}
    \foreach \x in {5,10,20,30}
      \draw[fdgrid] \P(\x,20) -- \P(\x,125);
    \foreach \y in {20,40,60,80,100,120}
      \draw[fdgrid] \P(0,\y) -- \P(34,\y);
    \draw[fdaxis] \P(0,20) -- \P(34,20);
    \draw[fdaxis] \P(0,20) -- \P(0,125);
    \foreach \x in {5,10,20,30}
      \draw[fdaxis] \P(\x,20) -- ++(0,-1.2mm) node[fdtick, below] {\x};
    \foreach \y in {20,40,60,80,100,120}
      \draw[fdaxis] \P(0,\y) -- ++(-1.2mm,0) node[fdtick, left] {\y};
    \draw[thick, DiagTrainL] \P(5,109.19) -- \P(10,98.36) -- \P(20,80.95) -- \P(30,67.28);
    \foreach \px/\py in {5/109.19, 10/98.36, 20/80.95, 30/67.28}
      \node[fdpt, fill=DiagTrainL] at \P(\px,\py) {};
    \draw[thick, DiagEmbL] \P(5,109.65) -- \P(10,98.21) -- \P(20,80.46) -- \P(30,65.19);
    \foreach \px/\py in {5/109.65, 10/98.21, 20/80.46, 30/65.19}
      \node[fdpt, fill=DiagEmbL] at \P(\px,\py) {};
    \draw[thick, DiagDataL] \P(5,109.53) -- \P(10,97.88) -- \P(20,79.73) -- \P(30,64.91);
    \foreach \px/\py in {5/109.53, 10/97.88, 20/79.73, 30/64.91}
      \node[fdpt, fill=DiagDataL] at \P(\px,\py) {};
    \draw[thick, NavyBlue] \P(5,109.76) -- \P(10,98.07) -- \P(20,79.26) -- \P(30,64.53);
    \foreach \px/\py in {5/109.76, 10/98.07, 20/79.26, 30/64.53}
      \node[fdpt, fill=NavyBlue] at \P(\px,\py) {};
    \draw[thick, DiagFrozenL] \P(5,109.71) -- \P(10,97.72) -- \P(20,78.91) -- \P(30,64.04);
    \foreach \px/\py in {5/109.71, 10/97.72, 20/78.91, 30/64.04}
      \node[fdpt, fill=DiagFrozenL] at \P(\px,\py) {};
    \draw[thick, DiagTrainL!55] \P(5,109.42) -- \P(10,97.72) -- \P(20,79.30) -- \P(30,62.71);
    \foreach \px/\py in {5/109.42, 10/97.72, 20/79.30, 30/62.71}
      \node[fdpt, fill=DiagTrainL!55] at \P(\px,\py) {};
    \draw[thick, DiagEmbL!55] \P(5,109.51) -- \P(10,97.82) -- \P(20,79.40) -- \P(30,62.86);
    \foreach \px/\py in {5/109.51, 10/97.82, 20/79.40, 30/62.86}
      \node[fdpt, fill=DiagEmbL!55] at \P(\px,\py) {};
    % legend
    \foreach \c/\yy/\lab in {DiagTrainL/121.0/1, DiagEmbL/116.1/3, DiagDataL/111.2/5, NavyBlue/106.3/10, DiagFrozenL/101.4/20, DiagTrainL!55/96.5/50, DiagEmbL!55/91.6/100} {
      \draw[thick, \c] \P(19.4,\yy) -- \P(22.1462,\yy);
      \node[fdpt, fill=\c] at \P(20.7731,\yy) {};
      \node[fdleg, right=0.8mm] at \P(22.1462,\yy) {$\kappa=\,$\lab};
    }
    \node[fdname, below=4.5mm] at \P(17,20) {weight $w_{\mathrm{div}}$};
    \node[fdname, rotate=90, anchor=south] at ({\fdx{0}-8.5mm},{\fdy{72.5}}) {MIND};
    \node[fdname, above=0.5mm] at \P(17,125) {AvgkNN};
  \end{tikzpicture}%
  \hfill
  \begin{tikzpicture}
    \def\fdx#1{((#1)-30)/290*52*1mm}
    \def\fdy#1{((#1)-20)/105*38*1mm}
    \def\P(#1,#2){({\fdx{#1}},{\fdy{#2}})}
    \foreach \x in {50,100,150,200,250,300}
      \draw[fdgrid] \P(\x,20) -- \P(\x,125);
    \foreach \y in {20,40,60,80,100,120}
      \draw[fdgrid] \P(30,\y) -- \P(320,\y);
    \draw[fdaxis] \P(30,20) -- \P(320,20);
    \draw[fdaxis] \P(30,20) -- \P(30,125);
    \foreach \x in {50,100,150,200,250,300}
      \draw[fdaxis] \P(\x,20) -- ++(0,-1.2mm) node[fdtick, below] {\x};
    \foreach \y in {20,40,60,80,100,120}
      \draw[fdaxis] \P(30,\y) -- ++(-1.2mm,0) node[fdtick, left] {\y};
    \draw[thick, DiagTrainL] \P(50,116.85) -- \P(100,112.19) -- \P(150,106.24) -- \P(200,101.33) -- \P(250,96.56) -- \P(300,93.03);
    \foreach \px/\py in {50/116.85, 100/112.19, 150/106.24, 200/101.33, 250/96.56, 300/93.03}
      \node[fdpt, fill=DiagTrainL] at \P(\px,\py) {};
    \draw[thick, DiagEmbL] \P(50,115.97) -- \P(100,109.52) -- \P(150,103.25) -- \P(200,98.26) -- \P(250,92.34);
    \foreach \px/\py in {50/115.97, 100/109.52, 150/103.25, 200/98.26, 250/92.34}
      \node[fdpt, fill=DiagEmbL] at \P(\px,\py) {};
    \draw[thick, DiagDataL] \P(50,120.74) -- \P(100,117.26) -- \P(150,114.93) -- \P(200,112.46) -- \P(250,109.68);
    \foreach \px/\py in {50/120.74, 100/117.26, 150/114.93, 200/112.46, 250/109.68}
      \node[fdpt, fill=DiagDataL] at \P(\px,\py) {};
    \draw[thick, NavyBlue] \P(50,120.34) -- \P(100,117.60) -- \P(150,115.28) -- \P(200,112.77) -- \P(250,109.73);
    \foreach \px/\py in {50/120.34, 100/117.60, 150/115.28, 200/112.77, 250/109.73}
      \node[fdpt, fill=NavyBlue] at \P(\px,\py) {};
    % legend
    \foreach \c/\yy/\lab in {DiagTrainL/52.0/0.2, DiagEmbL/45.5/0.4, DiagDataL/39.0/0.6, NavyBlue/32.5/0.8} {
      \draw[thick, \c] \P(45.95,\yy) -- \P(69.3731,\yy);
      \node[fdpt, fill=\c] at \P(57.6615,\yy) {};
      \node[fdleg, right=0.8mm] at \P(69.3731,\yy) {$\sigma_s=\,$\lab};
    }
    \node[fdname, below=4.5mm] at \P(175,20) {diversity weight $w_{\mathrm{div}}$};
    \node[fdname, rotate=90, anchor=south] at ({\fdx{30}-8.5mm},{\fdy{72.5}}) {MIND};
    \node[fdname, above=0.5mm] at \P(175,125) {PerturbProj};
  \end{tikzpicture}

  \vspace{7mm}

  \begin{tikzpicture}
    \def\fdx#1{((#1)-25)/295*52*1mm}
    \def\fdy#1{((#1)-20)/105*38*1mm}
    \def\P(#1,#2){({\fdx{#1}},{\fdy{#2}})}
    \foreach \x in {40,80,120,160,200,300}
      \draw[fdgrid] \P(\x,20) -- \P(\x,125);
    \foreach \y in {20,40,60,80,100,120}
      \draw[fdgrid] \P(25,\y) -- \P(320,\y);
    \draw[fdaxis] \P(25,20) -- \P(320,20);
    \draw[fdaxis] \P(25,20) -- \P(25,125);
    \foreach \x in {40,80,120,160,200,300}
      \draw[fdaxis] \P(\x,20) -- ++(0,-1.2mm) node[fdtick, below] {\x};
    \foreach \y in {20,40,60,80,100,120}
      \draw[fdaxis] \P(25,\y) -- ++(-1.2mm,0) node[fdtick, left] {\y};
    \draw[thick, DiagTrainL] \P(40,113.38) -- \P(80,102.23) -- \P(120,91.43) -- \P(160,86.16) -- \P(200,77.67) -- \P(300,64.90);
    \foreach \px/\py in {40/113.38, 80/102.23, 120/91.43, 160/86.16, 200/77.67, 300/64.90}
      \node[fdpt, fill=DiagTrainL] at \P(\px,\py) {};
    \draw[thick, DiagEmbL] \P(40,90.88) -- \P(80,64.45) -- \P(120,44.96) -- \P(160,33.61) -- \P(200,27.57);
    \foreach \px/\py in {40/90.88, 80/64.45, 120/44.96, 160/33.61, 200/27.57}
      \node[fdpt, fill=DiagEmbL] at \P(\px,\py) {};
    \draw[thick, DiagDataL] \P(40,95.00) -- \P(80,72.89) -- \P(120,57.84) -- \P(160,45.02) -- \P(200,35.70);
    \foreach \px/\py in {40/95.00, 80/72.89, 120/57.84, 160/45.02, 200/35.70}
      \node[fdpt, fill=DiagDataL] at \P(\px,\py) {};
    \draw[thick, NavyBlue] \P(40,100.96) -- \P(80,83.45) -- \P(120,71.76) -- \P(160,62.30) -- \P(200,54.99);
    \foreach \px/\py in {40/100.96, 80/83.45, 120/71.76, 160/62.30, 200/54.99}
      \node[fdpt, fill=NavyBlue] at \P(\px,\py) {};
    % legend
    \foreach \c/\yy/\lab in {DiagTrainL/121.0/0.2, DiagEmbL/116.1/0.4, DiagDataL/111.2/0.6, NavyBlue/106.3/0.8} {
      \draw[thick, \c] \P(202,\yy) -- \P(225.827,\yy);
      \node[fdpt, fill=\c] at \P(213.913,\yy) {};
      \node[fdleg, right=0.8mm] at \P(225.827,\yy) {$\sigma_s=\,$\lab};
    }
    \node[fdname, below=4.5mm] at \P(172.5,20) {diversity weight $w_{\mathrm{div}}$};
    \node[fdname, rotate=90, anchor=south] at ({\fdx{25}-8.5mm},{\fdy{72.5}}) {MIND};
    \node[fdname, above=0.5mm] at \P(172.5,125) {SGMS};
  \end{tikzpicture}%
  \hfill
  \begin{tikzpicture}
    \def\fdx#1{((#1)-30)/390*52*1mm}
    \def\fdy#1{((#1)-20)/105*38*1mm}
    \def\P(#1,#2){({\fdx{#1}},{\fdy{#2}})}
    \foreach \x in {50,150,250,350,400}
      \draw[fdgrid] \P(\x,20) -- \P(\x,125);
    \foreach \y in {20,40,60,80,100,120}
      \draw[fdgrid] \P(30,\y) -- \P(420,\y);
    \draw[fdaxis] \P(30,20) -- \P(420,20);
    \draw[fdaxis] \P(30,20) -- \P(30,125);
    \foreach \x in {50,150,250,350,400}
      \draw[fdaxis] \P(\x,20) -- ++(0,-1.2mm) node[fdtick, below] {\x};
    \foreach \y in {20,40,60,80,100,120}
      \draw[fdaxis] \P(30,\y) -- ++(-1.2mm,0) node[fdtick, left] {\y};
    \draw[thick, DiagTrainL] \P(50,98.28) -- \P(100,78.76) -- \P(150,73.87) -- \P(200,71.96) -- \P(250,71.17) -- \P(300,74.46) -- \P(350,75.89) -- \P(400,76.56);
    \foreach \px/\py in {50/98.28, 100/78.76, 150/73.87, 200/71.96, 250/71.17, 300/74.46, 350/75.89, 400/76.56}
      \node[fdpt, fill=DiagTrainL] at \P(\px,\py) {};
    \node[fdname, below=4.5mm] at \P(225,20) {diversity weight $w_{\mathrm{div}}$};
    \node[fdname, rotate=90, anchor=south] at ({\fdx{30}-8.5mm},{\fdy{72.5}}) {MIND};
    \node[fdname, above=0.5mm] at \P(225,125) {SigLIPMS};
  \end{tikzpicture}

  \caption{MIND against the diversity guidance weight $w_{\mathrm{div}}$, one curve per
  hyperparameter configuration. All four panels share the $\MIND$ axis. \code{avgKNN} is nearly insensitive to $\kappa$ over two
  orders of magnitude (seven superimposed curves, $\kappa=1$ to $\kappa=100$). For \code{SGMS}
  and \code{PerturbProj} the perturbation scale is non-monotone: $\sigma_s=0.2$ is
  markedly \emph{worse} than $0.4$, so the useful range is
  $\sigma_s\in[0.4,0.6]$. \code{SGMS} is the only method reaching $\MIND$ below $60$, and it does
  so at a precision below $0.9$ (\cref{tab:sgms-pr}).
  \code{SigLIPMS} turns around: $\MIND$ bottoms out at $w_{\mathrm{div}}=250$ and rises again
  beyond it.}
  \label{fig:mind-vs-weight}
\end{figure}

\begin{table}[htbp]
\begin{minipage}[t]{0.48\linewidth}\centering\footnotesize
\caption{\code{avgKNN} alone, 1,000 images: precision, recall and $\MIND$ per $(\kappa, w_{\mathrm{div}})$.}
\label{tab:avgknn-pr}
\vspace{\baselineskip}
\begin{tabular}{r r rrr}
    \toprule
    $\kappa$ & $w_{\mathrm{div}}$ & \hi{Precision} & \hi{Recall} & \lo{$\MIND$} \\
    \midrule
    \multirow{4}{*}{1} & 5 & \best{0.971} & \shade{19}{0.383} & \shade{33}{109.19} \\
      & 10 & \shade{59}{0.964} & \shade{37}{0.392} & \shade{37}{98.36} \\
      & 20 & \shade{37}{0.955} & \shade{78}{0.412} & \shade{59}{80.95} \\
      & 30 & \shade{19}{0.946} & \shade{93}{0.435} & \shade{82}{67.28} \\
    \midrule
    \multirow{4}{*}{3} & 5 & \shade{89}{0.969} & \shade{41}{0.394} & \shade{22}{109.65} \\
      & 10 & \shade{52}{0.963} & \shade{56}{0.403} & \shade{41}{98.21} \\
      & 20 & \shade{59}{0.964} & \shade{15}{0.381} & \shade{63}{80.46} \\
      & 30 & \shade{15}{0.944} & \shade{89}{0.433} & \shade{85}{65.19} \\
    \midrule
    \multirow{4}{*}{5} & 5 & \shade{89}{0.969} & \shade{33}{0.391} & \shade{26}{109.53} \\
      & 10 & \shade{74}{0.967} & \shade{59}{0.404} & \shade{48}{97.88} \\
      & 20 & \shade{45}{0.960} & \shade{45}{0.397} & \shade{67}{79.73} \\
      & 30 & \shade{26}{0.949} & \shade{96}{0.441} & \shade{89}{64.91} \\
    \midrule
    \multirow{4}{*}{10} & 5 & \shade{96}{0.970} & \shade{30}{0.389} & \shade{15}{109.76} \\
      & 10 & \shade{67}{0.965} & \shade{56}{0.403} & \shade{45}{98.07} \\
      & 20 & \shade{48}{0.961} & \shade{67}{0.407} & \shade{74}{79.26} \\
      & 30 & \shade{26}{0.949} & \best{0.449} & \shade{93}{64.53} \\
    \midrule
    \multirow{4}{*}{20} & 5 & \shade{89}{0.969} & \shade{26}{0.386} & \shade{19}{109.71} \\
      & 10 & \shade{89}{0.969} & \shade{48}{0.398} & \shade{56}{97.72} \\
      & 20 & \shade{45}{0.960} & \shade{78}{0.412} & \shade{78}{78.91} \\
      & 30 & \shade{30}{0.951} & \shade{82}{0.424} & \shade{96}{64.04} \\
    \midrule
    \multirow{4}{*}{50} & 5 & \shade{96}{0.970} & \shade{22}{0.384} & \shade{30}{109.42} \\
      & 10 & \shade{67}{0.965} & \shade{67}{0.407} & \shade{56}{97.72} \\
      & 20 & \shade{70}{0.966} & \shade{70}{0.408} & \shade{70}{79.30} \\
      & 30 & \shade{33}{0.954} & \shade{85}{0.425} & \best{62.71} \\
    \bottomrule
  \end{tabular}
\end{minipage}\hfill
\begin{minipage}[t]{0.48\linewidth}\centering\footnotesize
\caption{\code{SigLIPMS} alone, 1,000 images: precision, recall and $\MIND$ per $w_{\mathrm{div}}$.}
\label{tab:siglipms-pr}
\vspace{\baselineskip}
\begin{tabular}{r rrr}
    \toprule
    $w_{\mathrm{div}}$ & \hi{Precision} & \hi{Recall} & \lo{$\MIND$} \\
    \midrule
    50 & \best{0.965} & \shade{51}{0.342} & \shade{15}{98.28} \\
    100 & \shade{88}{0.957} & \shade{88}{0.382} & \shade{27}{78.76} \\
    150 & \shade{76}{0.954} & \best{0.405} & \shade{76}{73.87} \\
    200 & \shade{51}{0.952} & \shade{76}{0.349} & \shade{88}{71.96} \\
    250 & \shade{64}{0.953} & \shade{39}{0.332} & \best{71.17} \\
    300 & \shade{27}{0.944} & \shade{64}{0.345} & \shade{64}{74.46} \\
    350 & \shade{39}{0.948} & \shade{27}{0.316} & \shade{51}{75.89} \\
    400 & \shade{15}{0.941} & \shade{15}{0.299} & \shade{39}{76.56} \\
    \bottomrule
  \end{tabular}
\end{minipage}
\end{table}

\begin{table}[htbp]
\begin{minipage}[t]{0.48\linewidth}\centering\footnotesize
\caption{\code{PerturbProj} alone, 1,000 images: precision, recall and $\MIND$ per $(\sigma_s, w_{\mathrm{div}})$.}
\label{tab:pproj-pr}
\vspace{\baselineskip}
\begin{tabular}{r r rrr}
    \toprule
    $\sigma_s$ & $w_{\mathrm{div}}$ & \hi{Precision} & \hi{Recall} & \lo{$\MIND$} \\
    \midrule
    \multirow{6}{*}{0.2} & 50 & \shade{53}{0.970} & \shade{53}{0.364} & \shade{32}{116.85} \\
      & 100 & \shade{87}{0.973} & \shade{40}{0.354} & \shade{58}{112.19} \\
      & 150 & \shade{40}{0.967} & \shade{28}{0.336} & \shade{74}{106.24} \\
      & 200 & \shade{24}{0.959} & \shade{15}{0.322} & \shade{83}{101.33} \\
      & 250 & \shade{36}{0.965} & \shade{24}{0.330} & \shade{92}{96.56} \\
      & 300 & \shade{28}{0.963} & \shade{24}{0.330} & \shade{96}{93.03} \\
    \midrule
    \multirow{5}{*}{0.4} & 50 & \shade{70}{0.971} & \shade{49}{0.358} & \shade{36}{115.97} \\
      & 100 & \shade{53}{0.970} & \shade{74}{0.379} & \shade{70}{109.52} \\
      & 150 & \shade{32}{0.964} & \shade{87}{0.398} & \shade{79}{103.25} \\
      & 200 & \shade{15}{0.956} & \shade{70}{0.374} & \shade{87}{98.26} \\
      & 250 & \shade{19}{0.958} & \shade{79}{0.386} & \best{92.34} \\
    \midrule
    \multirow{5}{*}{0.6} & 50 & \best{0.974} & \shade{36}{0.351} & \shade{15}{120.74} \\
      & 100 & \shade{83}{0.972} & \shade{32}{0.350} & \shade{28}{117.26} \\
      & 150 & \shade{70}{0.971} & \shade{45}{0.356} & \shade{45}{114.93} \\
      & 200 & \shade{83}{0.972} & \shade{70}{0.374} & \shade{53}{112.46} \\
      & 250 & \shade{83}{0.972} & \shade{62}{0.365} & \shade{66}{109.68} \\
    \midrule
    \multirow{5}{*}{0.8} & 50 & \best{0.974} & \shade{62}{0.365} & \shade{19}{120.34} \\
      & 100 & \best{0.974} & \shade{83}{0.389} & \shade{24}{117.60} \\
      & 150 & \shade{70}{0.971} & \shade{92}{0.406} & \shade{40}{115.28} \\
      & 200 & \shade{45}{0.968} & \shade{96}{0.416} & \shade{49}{112.77} \\
      & 250 & \shade{70}{0.971} & \best{0.421} & \shade{62}{109.73} \\
    \bottomrule
  \end{tabular}
\end{minipage}\hfill
\begin{minipage}[t]{0.48\linewidth}\centering\footnotesize
\caption{\code{SGMS} alone, 1,000 images: precision, recall and $\MIND$ per $(\sigma_s, w_{\mathrm{div}})$.}
\label{tab:sgms-pr}
\vspace{\baselineskip}
\begin{tabular}{r r rrr}
    \toprule
    $\sigma_s$ & $w_{\mathrm{div}}$ & \hi{Precision} & \hi{Recall} & \lo{$\MIND$} \\
    \midrule
    \multirow{6}{*}{0.2} & 40 & \best{0.972} & \shade{15}{0.390} & \shade{15}{113.38} \\
      & 80 & \shade{96}{0.960} & \shade{28}{0.420} & \shade{19}{102.23} \\
      & 120 & \shade{79}{0.939} & \shade{19}{0.410} & \shade{32}{91.43} \\
      & 160 & \shade{70}{0.933} & \shade{24}{0.411} & \shade{40}{86.16} \\
      & 200 & \shade{62}{0.921} & \shade{32}{0.425} & \shade{49}{77.67} \\
      & 300 & \shade{49}{0.897} & \shade{36}{0.436} & \shade{62}{64.90} \\
    \midrule
    \multirow{5}{*}{0.4} & 40 & \shade{87}{0.947} & \shade{53}{0.475} & \shade{36}{90.88} \\
      & 80 & \shade{58}{0.916} & \shade{70}{0.508} & \shade{66}{64.45} \\
      & 120 & \shade{32}{0.856} & \shade{66}{0.499} & \shade{87}{44.96} \\
      & 160 & \shade{19}{0.820} & \shade{58}{0.481} & \shade{96}{33.61} \\
      & 200 & \shade{15}{0.766} & \shade{74}{0.519} & \best{27.57} \\
    \midrule
    \multirow{5}{*}{0.6} & 40 & \shade{83}{0.946} & \shade{53}{0.475} & \shade{28}{95.00} \\
      & 80 & \shade{66}{0.923} & \shade{87}{0.530} & \shade{53}{72.89} \\
      & 120 & \shade{49}{0.897} & \shade{79}{0.528} & \shade{74}{57.84} \\
      & 160 & \shade{36}{0.859} & \shade{96}{0.540} & \shade{83}{45.02} \\
      & 200 & \shade{24}{0.849} & \best{0.567} & \shade{92}{35.70} \\
    \midrule
    \multirow{5}{*}{0.8} & 40 & \shade{92}{0.955} & \shade{45}{0.468} & \shade{24}{100.96} \\
      & 80 & \shade{74}{0.938} & \shade{40}{0.460} & \shade{45}{83.45} \\
      & 120 & \shade{53}{0.907} & \shade{62}{0.495} & \shade{58}{71.76} \\
      & 160 & \shade{40}{0.874} & \shade{83}{0.529} & \shade{70}{62.30} \\
      & 200 & \shade{28}{0.855} & \shade{92}{0.531} & \shade{79}{54.99} \\
    \bottomrule
  \end{tabular}
\end{minipage}
\end{table}

Three regularities hold across the four terms. \code{avgKNN} is nearly insensitive to its
neighbourhood size: at its strongest weight, $\MIND$ varies by $4.6$ across seven values of $\kappa$
from $1$ to $100$, against a $46$-point swing driven by the weight alone. For \code{SGMS} and
\code{PerturbProj} the re-noising level is non-monotone, $\sigma_s = 0.2$ being markedly worse than
$0.4$, so the useful range is $\sigma_s \in [0.4, 0.6]$ for \code{SGMS}, while \code{PerturbProj}
barely leaves the model's manifold at any setting (precision never below $0.956$) and only
$\sigma_s = 0.8$ gains recall. \code{SGMS} is the only term that converts precision into recall at
scale in this setting, which is also why it is the one term whose low-$\MIND$ settings fail the
precision floor of \cref{tab:diversity-alone-mind}; \code{SigLIPMS} has a turning point: $\MIND$
bottoms out at $w_{\mathrm{div}} = 250$ and rises beyond it while recall falls from its peak at
$150$, so weights above $250$ are harmful.

\subsection{Qualitative results}
\label{app:qualitative}

Every figure of this section shares its seeds across rows, so a column is one starting noise under
several regimes and the panels can be read image by image. Together with \cref{fig:teaser} they
show three things: the guidance acts on individual images and leaves what it does not name, our
method changes a face more cleanly than a labelled predictor, and the wording of the sentences
matters little.

\paragraph{CelebA-HQ against Balancing Act (\cref{fig:celebahq-gender}).}
Six seeds, three of them women in the unguided row, under Balancing Act and under \method{} at
matched discrepancy. Both methods turn the same two columns, the fifth and sixth, into men, and both
leave the first three women and the man in the fourth column as they are, so the two batches end at
the same proportion by the same choice of images. The difference is in the faces they produce: the
man Balancing Act makes in the fifth column carries a blurred, uneven skin and a distorted mouth,
while ours is a clean face with the seed's hair and lighting, and the sixth column shows the same
gap in a milder form. That is the visible side of the $\FID$ difference in \cref{tab:celebahq}: at
equal discrepancy, guidance through a labelled predictor pays in image quality where guidance
through the adapter does not.

\paragraph{Sensitivity to the sentences (\cref{fig:celebahq-prompts}).}
The same six seeds and the same weight $w_{\text{fair}} = 20$, with five different pairs of
sentences naming gender: woman and man, female and male person, girl's and boy's portrait, lady and
gentleman, and close-up photos of a woman and of a man. All five pairs turn the same two columns,
the fifth and sixth, and leave the four others untouched; the five men in the fifth column differ
only in fine detail, a lock of hair or the shape of the smile. The attribute the guidance reads is
the direction between the two sentences, and that direction hardly moves with the wording, so an
operator need not search for the right words.

\paragraph{Stable Diffusion (\cref{fig:sd-qual}).}
Six seeds each for \emph{A face of a doctor} and \emph{A face of a firefighter}, unguided and under
\method{} at a gender target of one half. The unguided model draws five men and one woman for the
doctor and six men for the firefighter, the skew of \cref{tab:t2i}. Guidance turns two doctor
columns and three firefighter columns into women, reaching three of six in each, and takes the
images whose logit was already closest: the woman of the sixth doctor column is left alone. White
coats and stethoscopes, helmets and turnout gear, the framing and the expression of each seed
stay; only the face changes, and with it hair and, on two firefighters, a hood under the helmet.
These are the samples behind the gender row of \cref{tab:t2i}.

\begin{figure}[htbp]\centering
\rowfig{0.94\linewidth}{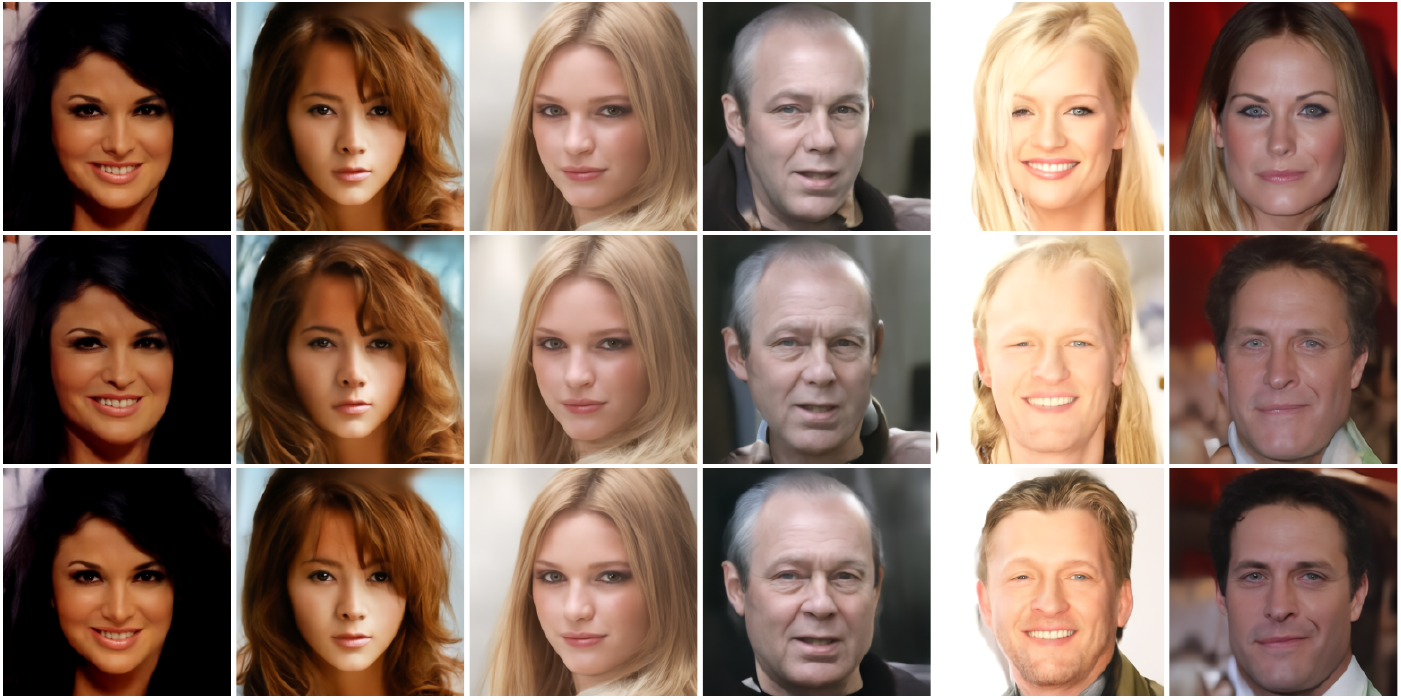}{0}{0.833/{Original}, 0.500/{Balancing Act}, 0.167/{\method{}}}
\caption{Gender on CelebA-HQ, six shared seeds: the unguided model, Balancing Act, and \method{}
at matched discrepancy.}
\label{fig:celebahq-gender}
\end{figure}

\begin{figure}[htbp]\centering
\includegraphics[width=0.85\linewidth]{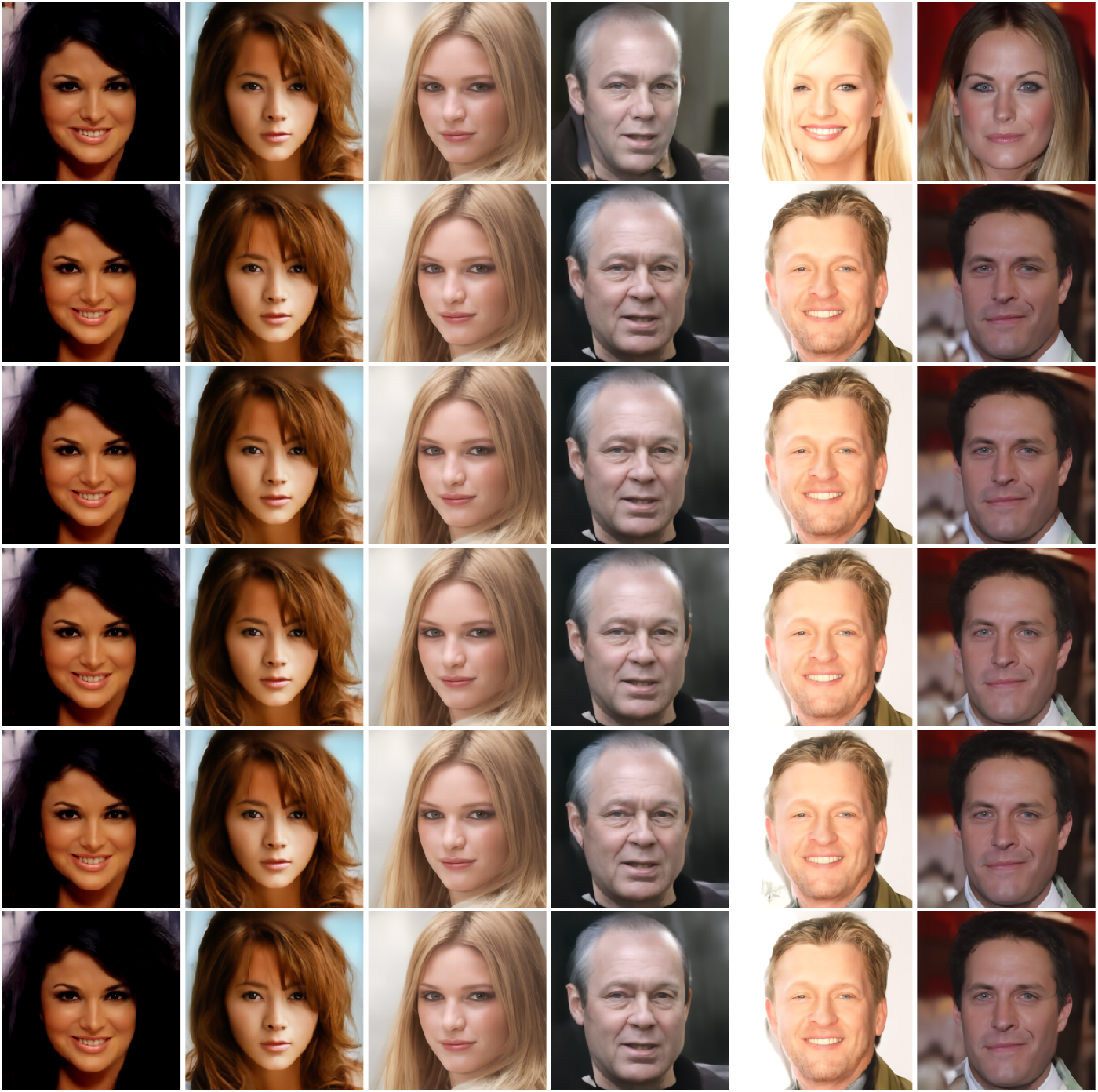}
\caption{Sensitivity to the prompt pair, gender on CelebA-HQ, $w_{\text{fair}} = 20$, six shared
seeds. Top: unguided. Then, one pair per row: \emph{a photo of a woman} / \emph{a photo of a man};
\emph{a female person} / \emph{a male person}; \emph{a girl's portrait} / \emph{a boy's portrait};
\emph{a photo of a lady} / \emph{a photo of a gentleman}; \emph{a woman's close-up photo} /
\emph{a man's close-up photo}.}
\label{fig:celebahq-prompts}
\end{figure}

\begin{figure}[htbp]\centering
\rowfig{0.94\linewidth}{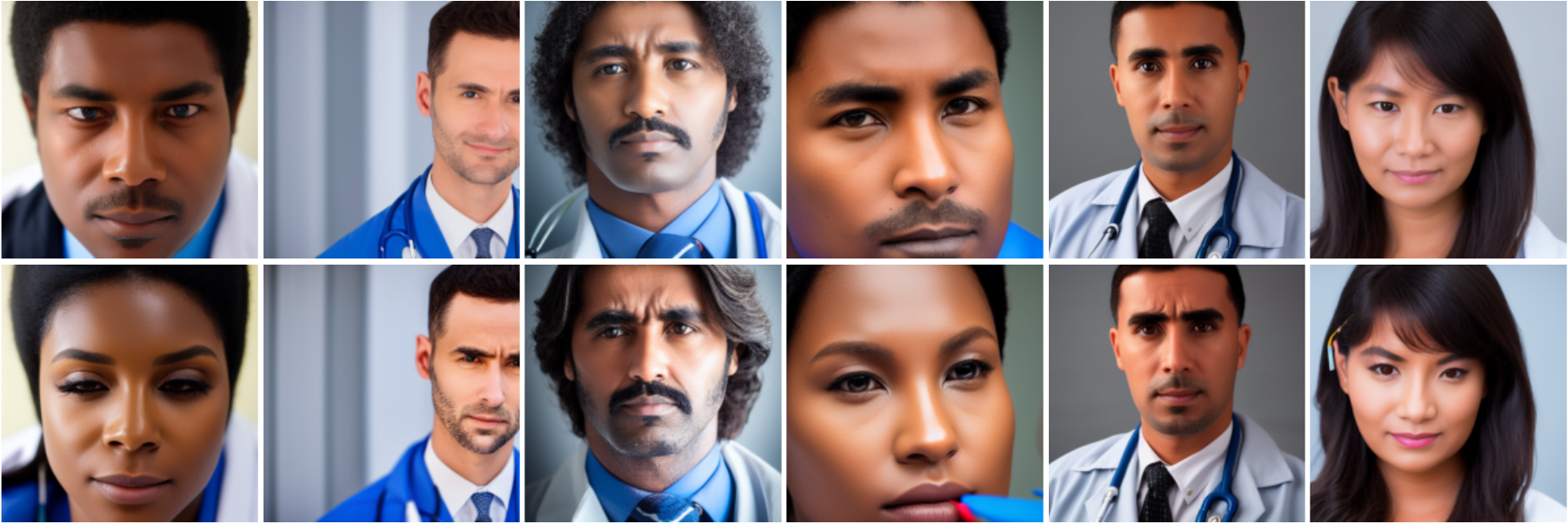}{0}{0.753/{Original}, 0.248/{\method{}}}\\[3pt]
\rowfig{0.94\linewidth}{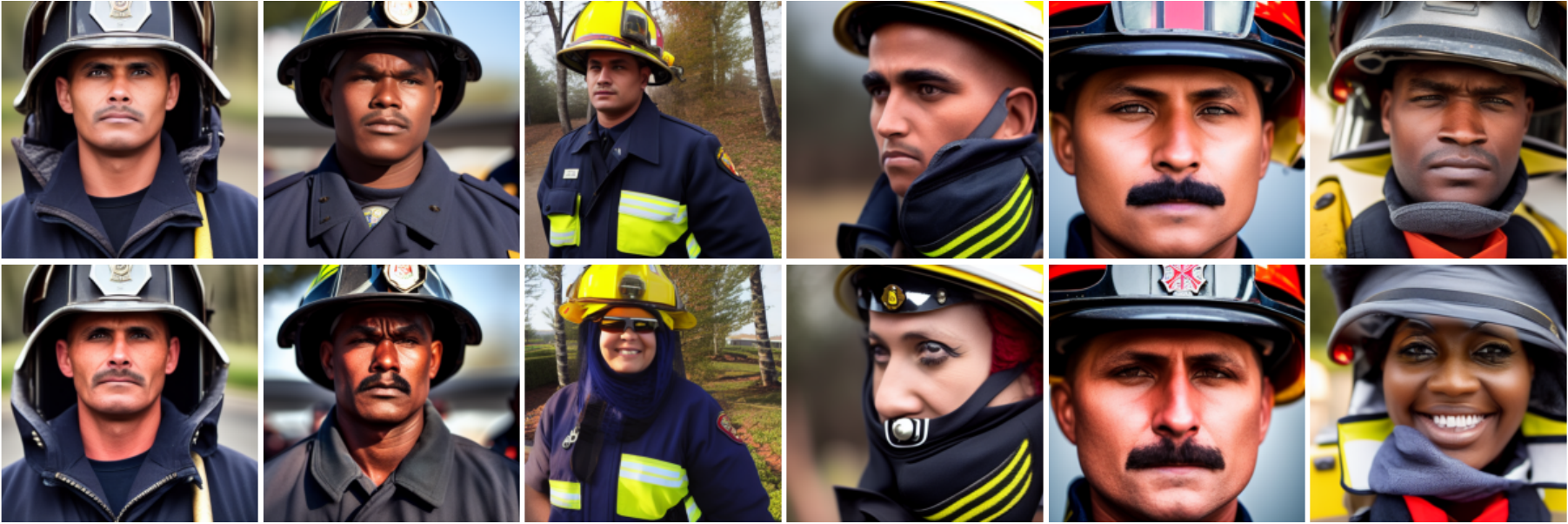}{0}{0.753/{Original}, 0.248/{\method{}}}
\caption{Stable Diffusion~1.5, prompts \emph{A face of a doctor} (top pair) and \emph{A face of a
firefighter} (bottom pair), six shared seeds each: the unguided model and \method{} at a gender
target of one half. Two doctor columns and three firefighter columns change gender; clothing, pose,
background and expression stay.}
\label{fig:sd-qual}
\end{figure}
% ===========================================================================

\section{Full comparison under the DiffLens protocol}
\label{app:t2i-full}

\cref{tab:t2i-full} completes \cref{tab:t2i} with the two CLIP scores of
\cref{app:details-metrics}, and the last paragraph of this section traces our residual
discrepancy to its sources. CLIP-T lies within $0.01$ of the unguided model for every method on
every attribute, so none of them, ours included, trades prompt adherence for the attribute. CLIP-I
separates the methods into two groups: the label-located methods \citep{shi2025difflens,zaleska2026icm}
and fine-tuning \citep{shen2024finetuning} keep $0.85$ to $0.92$ of the original, the latent-space
editors \citep{kwon2023asyrp,li2024selfdiscovering} $0.81$ to $0.88$, and ours sits at $0.84$ to
$0.85$, the level of the $h$-space distribution guidance \citep{parihar2024balancing}. Race is the
attribute on which DiffLens and ICM keep the most of the original image, above $0.90$, and also the
one on which they move the discrepancy least.

\paragraph{Where the residual discrepancy comes from.} The $\FD$ of \cref{tab:t2i} is an average
over the four occupations of the protocol, and \cref{tab:t2i-perjob} gives the four terms. On
gender they are far apart: doctors reach $0.021$ and firefighters $0.104$, nurses stay at $0.296$
and receptionists at $0.378$. For a binary attribute $\FD = \sqrt{2}\,|p - \tfrac12|$
(\cref{app:details-metrics}), so these are $51\%$, $57\%$, $71\%$ and $77\%$ women. The unguided
model explains the split. \citet{seshadri2024amplification} measured the share of women that Stable
Diffusion~1.4 draws for each occupation under four prompt templates: $100\%$ for \emph{nurse} under
every template and $99.8$ to $100\%$ for \emph{receptionist}, against $1$ to $16\%$ for
\emph{firefighter} and $4$ to $58\%$ for \emph{doctor}. These shares account for the unguided row
of \cref{tab:t2i}: an all-female or all-male batch scores $0.71$, and averaging the four templates
gives per-occupation discrepancies of $0.71$, $0.71$, $0.63$ and $0.32$, whose mean of $0.59$ is
close to the $0.56$ we measure on Stable Diffusion~1.5 with the prompt of the protocol. Reaching one
half means creating male nurses and receptionists that the model never draws on its own, whereas
female firefighters and doctors exist in its output already; the adapter, trained on the model's
own images of these jobs, has seen just as few. Guidance moves nurses and receptionists from all
women to three quarters, and stops there. The other guidance method of \cref{tab:t2i}, the
$h$-distribution guidance of \citet{parihar2024balancing}, is trained on gender labels and stops at
$0.222$, above ours; the editing methods pass this point because they change what the model
computes, whereas guidance reweights what it draws.

On age the residual is concentrated in the same way. Doctors and nurses reach $0.092$ and $0.077$,
firefighters and receptionists stay at $0.276$ and $0.381$. The unguided model puts nearly every
face in one bin, its age discrepancy of $0.752$ being close to the $0.82$ of a single-bin batch,
and its firefighters and receptionists are, as far as its samples show, middle-aged men and young
women; a third of either at $60+$ is far from what it draws on its own. A second effect may add a
floor: FairFace counts $0$--$19$, $20$--$59$ and $60+$, while the guidance is given \emph{a young},
\emph{a middle-aged} and \emph{an old person}, and a face the encoder moves towards \emph{young} is
most likely in its twenties, which the counter files under $20$--$59$. The two low values on doctor
and nurse show that the sentences do reach the counted bins where the model can follow. On race the
residual is even across occupations, $0.17$ to $0.21$, so it belongs to the attribute and not to
any job; the four sentences name the four groups the counter uses, and ours is the lowest race
discrepancy in \cref{tab:t2i}.

\begin{table}[t]\centering\footnotesize
\setlength{\tabcolsep}{6pt}
\caption{\method{} on Stable Diffusion~1.5, discrepancy per occupation. Each entry is the distance
of the mean FairFace output to the uniform target on the $500$ images of one prompt; the last row
is their average, the $\FD$ of \cref{tab:t2i}.}
\label{tab:t2i-perjob}
\vspace{\baselineskip}
\begin{tabular}{lrrr}
\toprule
Occupation & \lo{$\FD$ gender} & \lo{$\FD$ age} & \lo{$\FD$ race} \\
\midrule
doctor & 0.021 & 0.092 & 0.208 \\
firefighter & 0.104 & 0.276 & 0.193 \\
nurse & 0.296 & 0.077 & 0.206 \\
receptionist & 0.378 & 0.381 & 0.174 \\
\midrule
mean (\cref{tab:t2i}) & 0.200 & 0.206 & 0.195 \\
\bottomrule
\end{tabular}
\end{table}

\begin{table}[t]\centering\footnotesize
\setlength{\tabcolsep}{3pt}
\caption{Stable Diffusion~1.5 \citep{rombach2022ldm} under the protocol of DiffLens
\citep{shi2025difflens}, all four metrics. Baseline rows other than ICM \citep{zaleska2026icm} are
as reported by \citet{shi2025difflens}, which gives fine-tuning \citep{shen2024finetuning} a single
$\FID$, CLIP-I and CLIP-T for the three attributes. \Legendbest; equal values share a shade. The
unguided row has no CLIP-I by definition.}
\label{tab:t2i-full}
\vspace{\baselineskip}
\resizebox{\linewidth}{!}{%
\begin{tabular}{llrrrrrrrrrrrr}
\toprule
& & \multicolumn{4}{c}{Gender} & \multicolumn{4}{c}{Age} & \multicolumn{4}{c}{Race} \\
\cmidrule(lr){3-6}\cmidrule(lr){7-10}\cmidrule(lr){11-14}
Method & Attribute & \lo{$\FD$} & \lo{$\FID$} & \hi{CLIP-I} & \hi{CLIP-T} & \lo{$\FD$} & \lo{$\FID$} & \hi{CLIP-I} & \hi{CLIP-T} & \lo{$\FD$} & \lo{$\FID$} & \hi{CLIP-I} & \hi{CLIP-T} \\
\midrule
Original & --- & 0.564 & 120.06 & -- & 0.6155 & 0.752 & 120.06 & -- & 0.6155 & 0.558 & 120.06 & -- & 0.6155 \\
\midrule
\multicolumn{14}{@{}l}{\emph{Feature or weight editing}} \\
Latent editing \citep{kwon2023asyrp} & text, training per attribute & \shade{15}{0.408} & \shade{15}{166.11} & \shade{29}{0.8253} & \shade{15}{0.6005} & \shade{29}{0.682} & \shade{15}{200.90} & \shade{58}{0.8527} & \shade{86}{0.6122} & \shade{29}{0.524} & \shade{29}{153.05} & \shade{72}{0.8804} & \shade{15}{0.6086} \\
Latent direction \citep{li2024selfdiscovering} & text, training per attribute & \shade{29}{0.305} & \shade{58}{129.37} & \shade{15}{0.8058} & \shade{86}{0.6091} & \shade{86}{0.052} & \shade{72}{113.81} & \shade{15}{0.8151} & \shade{43}{0.6067} & \best{0.175} & \shade{43}{128.30} & \shade{15}{0.8211} & \shade{72}{0.6132} \\
Fine-tuning \citep{shen2024finetuning} & attribute classifier & \shade{86}{0.050} & \shade{29}{161.47} & \best{0.8779} & \best{0.6095} & \shade{15}{0.746} & \shade{29}{161.47} & \shade{86}{0.8779} & \shade{58}{0.6095} & \shade{72}{0.198} & \shade{15}{161.47} & \shade{58}{0.8779} & \shade{29}{0.6095} \\
DiffLens \citep{shi2025difflens} & labels & \best{0.046} & \best{112.83} & \shade{86}{0.8501} & \shade{72}{0.6090} & \best{0.049} & \best{99.17} & \shade{72}{0.8778} & \shade{15}{0.6057} & \shade{43}{0.401} & \shade{72}{119.86} & \shade{86}{0.9096} & \shade{86}{0.6149} \\
ICM \citep{zaleska2026icm} & labels & \shade{72}{0.087} & \shade{72}{122.08} & \shade{86}{0.8501} & \shade{72}{0.6090} & \shade{72}{0.133} & \shade{58}{114.87} & \best{0.9195} & \best{0.6150} & \shade{58}{0.266} & \shade{86}{116.88} & \best{0.9099} & \best{0.6172} \\
\midrule
\multicolumn{14}{@{}l}{\emph{Guidance on the model's distribution}} \\
$h$-distribution \citep{parihar2024balancing} & labels & \shade{43}{0.222} & \shade{43}{151.68} & \shade{58}{0.8475} & \shade{29}{0.6087} & \shade{43}{0.506} & \shade{43}{147.71} & \shade{29}{0.8345} & \shade{72}{0.6098} & \shade{15}{0.544} & \shade{58}{126.90} & \shade{29}{0.8255} & \shade{43}{0.6100} \\
\method{} & text & \shade{58}{0.200} & \shade{86}{113.85} & \shade{43}{0.8381} & \shade{72}{0.6090} & \shade{58}{0.206} & \shade{86}{110.55} & \shade{43}{0.8409} & \shade{29}{0.6060} & \shade{86}{0.195} & \best{114.71} & \shade{43}{0.8527} & \shade{58}{0.6111} \\
\bottomrule
\end{tabular}}
\end{table}

\end{document}